\documentclass{article}

\PassOptionsToPackage{numbers}{natbib}
\usepackage[main, final]{neurips_2026}

\usepackage[utf8]{inputenc} %
\usepackage[T1]{fontenc}    %
\usepackage[colorlinks, allcolors=gray]{hyperref}       %
\usepackage{url}            %
\usepackage{booktabs}       %
\usepackage{amsfonts}       %
\usepackage{nicefrac}       %
\usepackage{microtype}      %
\usepackage{xcolor}         %
\usepackage{amsmath}
\usepackage{makecell} 
\usepackage{amssymb}
\usepackage{algorithmicx}
\usepackage{algcompatible}
\usepackage{algorithm}
\usepackage{caption}
\usepackage{array}
\usepackage{setspace}
\usepackage{cancel}
\usepackage{caption}
\usepackage{mathtools}
\usepackage{amsthm}
\usepackage{booktabs}
\usepackage{multirow}
\usepackage{ dsfont }
\usepackage{subcaption}
\usepackage{enumitem}
\usepackage{hyperref}
\usepackage{url}
\usepackage[capitalize,noabbrev]{cleveref}
\usepackage{bm}
\usepackage{bbm}
\theoremstyle{plain}
\newtheorem{theorem}{Theorem}[section]
\newtheorem{proposition}[theorem]{Proposition}
\newtheorem{lemma}[theorem]{Lemma}
\newtheorem{corollary}[theorem]{Corollary}
\theoremstyle{definition}

\newtheorem{assumption}[theorem]{Assumption}
\theoremstyle{remark}
\newtheorem{remark}[theorem]{Remark}

\definecolor{deepblue}{rgb}{0.12, 0.23, 0.45}

\newcommand{\f}{f}

\newcommand{\dd}{\mathrm{d}}

\newcommand{\bx}{\textbf{x}}
\newcommand{\K}{\textbf{K}}

\newcommand{\R}{\mathbb{R}}
\newcommand{\E}{\mathbb{E}}

\title{Conditioning Gaussian Processes on Almost Anything}

\author{%
  Henry B. Moss \thanks{Joint first author. Correspondence to \texttt{henry.moss@lancaster.ac.uk} and \texttt{l.astfalck@unsw.edu.au}} \\ Lancaster University
  \\
\And
  Lachlan Astfalck\textsuperscript{*} \\University of New South Wales
\And
  Thomas Cowperthwaite \\University of Cambridge
\And
  Colin Doumont \\University of Tübingen\\Tübingen AI Center 
\And
  Sam Willis \\University of Cambridge
\And
  Philipp Hennig \\University of Tübingen\\Tübingen AI Center
\And
  Christopher Nemeth \\Lancaster University
\And
  Andrew Zammit-Mangion \\University of New South Wales
}

\begin{document}

\maketitle

\begin{abstract}

Gaussian processes (GPs) offer a principled probabilistic model over functions, but exact inference is restricted to the linear-Gaussian regime. We establish an explicit equivalence between GPs and a class of linear diffusion models, recasting predictive sampling as an ODE with closed-form Gaussian dynamics and a likelihood-dependent guidance term that admits a simple Monte Carlo approximation. In the linear-Gaussian setting, we recover standard GP conditioning exactly; beyond conjugacy, the same machinery handles any conditioning statement admitting point-wise likelihood evaluation --- including non-linear physics, and, for the first time, natural language via large language models. Whitening isolates the irreducible non-Gaussian dynamics, minimising Wasserstein-2 transport cost and eliminating numerical stiffness. The result is a general-purpose GP inference scheme requiring no bespoke derivations. Together, these results provide a general mechanism for incorporating the full richness of real-world knowledge as conditioning information, opening a new frontier for the probabilistic modelling of real-world problems.

\end{abstract}

\section{Introduction} \label{intro}

Gaussian processes (GPs) provide a fully specified probabilistic model over functions with closed-form conditioning under linear Gaussian observations, offering a rare combination of flexibility and mathematical exactness \citep{williams2006gaussian}. Beyond the linear-Gaussian regime, exact conditioning is lost: non-Gaussian likelihoods and non-linear constraints force reliance on approximate inference schemes such as Laplace approximations \citep{rue2007approximate}, expectation propagation \citep{cseke2011approximate}, or variational methods \citep{hensman2013gaussian}, which are often computationally demanding, sensitive to parametrisation, and difficult to generalise across data modalities. This motivates a framework for conditional GP sampling that treats the problem in full generality, extending beyond conjugacy without bespoke approximations.

Our perspective draws on recent advances in diffusion modelling \citep{ho2020denoising, songscore} and flow-matching \citep{lipmanflow, albergo2023building, liu2023flow}, which transport samples from a simple reference distribution (typically Gaussian) through a sequence of intermediate distributions to match a target law \citep{lai2025principles}.  In diffusion modelling, this transport is described via the reverse-time SDE of a noising process \citep{songscore}; in flow-matching, a learnt velocity field induces an ODE pushing the reference to the target \citep{lipmanflow}. In both cases, sampling is realised as an evolution along a path in distribution space. This has led to the realisation that conditioning can be achieved by modifying the prior sampling dynamics at
test time so that trajectories are biased towards regions satisfying the conditioning statement \citep{ho2022classifier}.
Exactness is contingent on three approximations: (i) accuracy of the learned score, (ii) quality of the conditional score along the path, and (iii) discretisation error from the SDE/ODE solver. 
Function space extensions exist, for example FunDiff \citep{wang2026fundiff} and DDO \citep{lim2025score}; however, these still train a neural approximation of the score and remain subject to the same approximation errors.

The main contribution of this work is the realisation that GPs are an exact closed-form instance of diffusion and flow-matching models (see Section~\ref{subsec:flows}). General GP conditioning on arbitrary conditions may then be written as the modification of an underlying analytical transport, and existing conditioning techniques from the diffusion literature can be incorporated (Section~\ref{sec:nonconjugate}). Of the three sources of error mentioned above, the first is eliminated entirely: the GP prior is available in closed form and no score network need be learnt. We prove that the remaining two sources, Monte Carlo guidance estimation and solver discretisation, are independently controllable, yielding predictive samples at prescribed fidelity with a fixed per-sample cost and without the mixing-time dependencies of MCMC or the parametrisation sensitivities of variational inference.

\emph{Our contributions offer a new frontier for the GP modelling of practical real-world problems:}  
\begin{itemize}
\item \textbf{GP--diffusion equivalence.} We establish an explicit equivalence between GPs and a class of linear diffusion models, providing an alternative transport view of GP sampling.
\item \textbf{Conditioning on anything.} We show that standard linear GP conditioning arises as a special case of test-time guidance, and extend this to arbitrary non-linear, non-Gaussian conditioning via a Monte Carlo approximation requiring only point-wise likelihood evaluations.
\item \textbf{Whitening is the optimal parameterisation.} We show whitening minimises Wasserstein-2 transport cost, with numerical stability governed by the covariance condition number.
\item \textbf{A unified GP sampler.} We apply our sampler (\textsc{FlowGP}) across applications including constrained regression, physics-informed modelling, Bayesian optimisation, and conditioning GPs on natural language via large language models (see Figure \ref{fig:llm+physc}).  
\item \textbf{Computational efficiency.} \textsc{FlowGP} runtimes range from milliseconds for monotonic bounded regression, to less than 5 seconds for LLM- and physics-constrained generation. Code is available at \href{https://github.com/FlowInference/flow-inference}{\texttt{github.com/FlowInference/flow-inference}}.%
\end{itemize}

\begin{figure}[thbp]
\begin{minipage}[b]{\textwidth}
\centering

    \begin{subfigure}[b]{0.49\textwidth}
        \centering
        \includegraphics[trim={0 0 0 0}, clip, width=\textwidth]{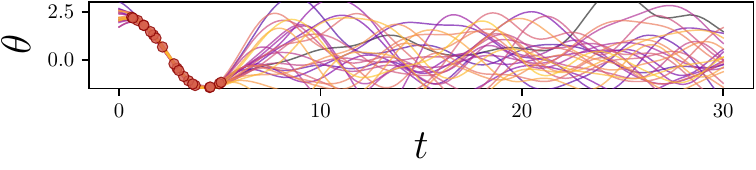}
    \end{subfigure}
    \hfill
    \begin{subfigure}[b]{0.49\textwidth}
        \centering
        \includegraphics[width=\textwidth]{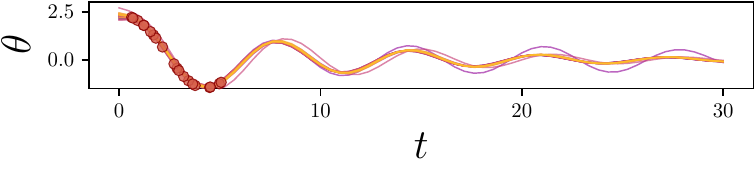}
    \end{subfigure}
   \textbf{(a)} \textsc{FlowGP} enforcing a damped pendulum equation constraint: $\frac{\mathrm{d}^2 \theta}{\mathrm{d}t^2} + \sin\!\theta + \beta\,\frac{\mathrm{d}\theta}{\mathrm{d}t} = 0$
\end{minipage}

\begin{minipage}[b]{\textwidth}
\centering
    \begin{subfigure}[b]{0.49\textwidth}
        \centering
        \includegraphics[trim={0 0 0 0}, clip,width=\textwidth]{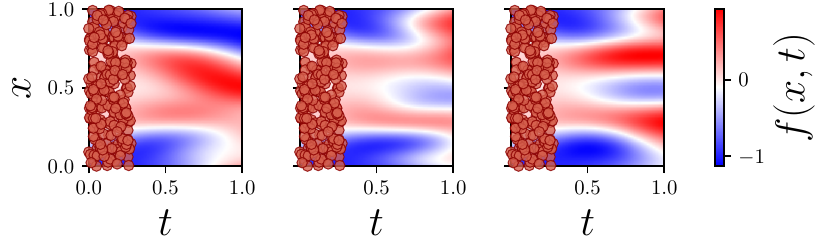}
    \end{subfigure}
    \hfill
    \begin{subfigure}[b]{0.49\textwidth}
        \centering
        \includegraphics[trim={0 0 0 0}, clip,width=\textwidth]{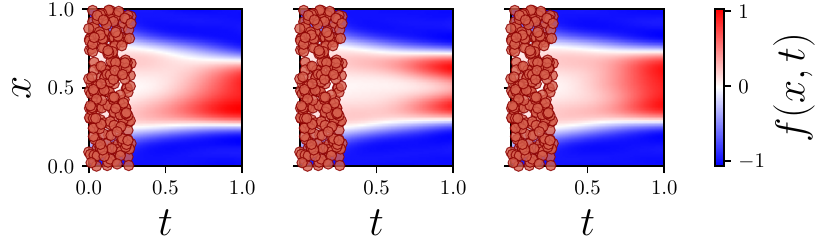}
    \end{subfigure}
\textbf{(b)} \textsc{FlowGP} enforcing an Allen--Cahn equation constraint: $\frac{\partial f}{\partial t}= \varepsilon \frac{\partial^2 f}{\partial x^2}+ 5f - 5f^3$
\end{minipage}

\begin{minipage}[b]{\textwidth}
    \centering
        \begin{subfigure}[b]{0.49\textwidth}
        \centering
        \includegraphics[trim={0 0 0 0}, clip,width=\textwidth]{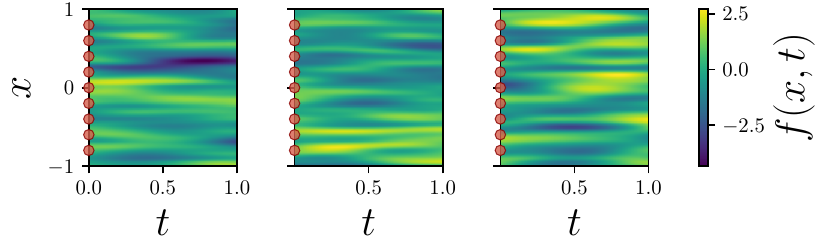}
    \end{subfigure}
    \hfill
    \begin{subfigure}[b]{0.49\textwidth}
        \centering
        \includegraphics[trim={0 0 0 0}, clip, width=\textwidth]{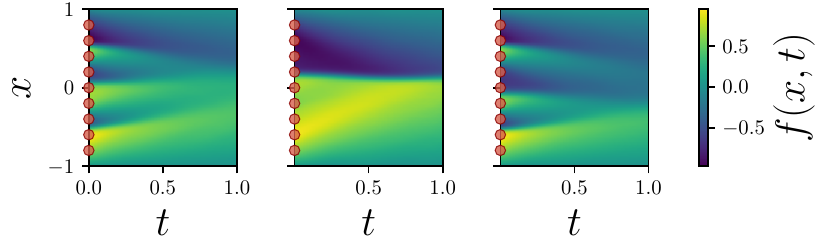}
    \end{subfigure}
\textbf{(c)} \textsc{FlowGP} enforcing a Burgers' equation constraint: $\frac{\partial f}{\partial t} + u\frac{\partial f}{\partial x}
    = \nu \frac{\partial^2 f}{\partial x^2}$
\end{minipage}

\begin{minipage}[b]{0.34\textwidth}
    \centering
    \begin{subfigure}[b]{0.54\textwidth}
        \includegraphics[width=\textwidth]{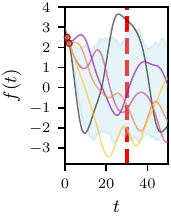}
    \end{subfigure}%
    \hfill
    \begin{subfigure}[b]{0.45\textwidth}
        \includegraphics[trim={0.5cm 0 0 0}, clip, width=\textwidth]{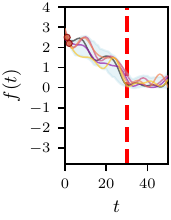}
    \end{subfigure}
\textbf{(d)} \textit{``...stock price...compulsory liquidation on day 30...''}
\end{minipage}%
\hfill
\begin{minipage}[b]{0.31\textwidth}
    \centering
    \begin{subfigure}[b]{0.49\textwidth}
        \includegraphics[trim={0.5cm 0 0 0}, clip, width=\textwidth]{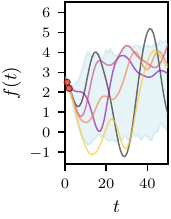}
    \end{subfigure}%
    \hfill
    \begin{subfigure}[b]{0.49\textwidth}
        \includegraphics[trim={0.5cm 0 0 0}, clip, width=\textwidth]{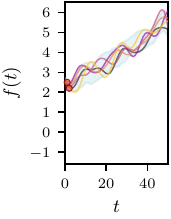}
    \end{subfigure}
\textbf{(e)} \textit{``...stock price...performs well...ending over 5 GBP...''}
\end{minipage}%
\hfill
\begin{minipage}[b]{0.31\textwidth}
    \centering
    \begin{subfigure}[b]{0.49\textwidth}
        \includegraphics[trim={0.5cm 0 0 0}, clip, width=\textwidth]{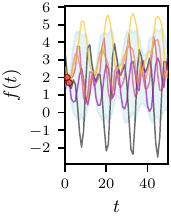}
    \end{subfigure}%
    \hfill
    \begin{subfigure}[b]{0.49\textwidth}
        \includegraphics[trim={0.5cm 0 0 0}, clip, width=\textwidth]{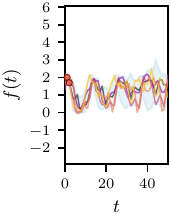}
    \end{subfigure}
\textbf{(f)} \textit{``Monthly average precipitation...January...''}
\end{minipage}
\caption{\textbf{(left of each pair)} Samples from a GP conditioned on observations (red dots) and \textbf{(right of each pair)} samples from \textsc{FlowGP} including additional information about non-linear physics  via known differential equations \textbf{(a-c)} and natural language descriptions via an LLM-based likelihood \textbf{(d-f)}. In each case, the unconstrained GP produces statistically coherent but semantically uninformed samples, whilst \textsc{FlowGP} produces samples that are also faithful to known physics or semantic content of the text prompt. See Section \ref{sec:experiments} for full experimental details.}
\label{fig:llm+physc}
\end{figure}

\section{A primer on Gaussian processes}\label{sec:GPs}

Let  $\f \in \mathcal{H}$ be a random function in some function space $\mathcal{H}$ indexed by locations $\bx \in \mathcal{X}\subseteq\mathbb{R}^d$. The defining property of a GP is that for any finite collection of $m$ test inputs $\textbf{X}_* \in\mathcal{X}^{m}$, the random vector $\bm f_0\in\mathbb{R}^m$ of evaluations of $f(\cdot)$ at $\textbf{X}_*$ follows a multivariate Gaussian distribution $\bm f_0\mid \textbf X_* \sim \mathcal{N}(\textbf{m}_*, \textbf{K}_{**})$, where $\textbf{m}_*=\mu(\textbf X_*)\in\mathbb{R}^m$ and covariance matrix $\textbf{K}_{**}:=k(\textbf{X}_*, \textbf{X}_*)\in\mathbb{R}^{m\times m}$ are induced by a mean function $\mu: \mathcal{X} \to \mathbb{R}$ and kernel function $k: \mathcal{X} \times \mathcal{X} \to \mathbb{R}$. We depart from the conventional notation of $\bm f_*$, instead using $\bm f_0$ to reflect both (i) the standard statistical interpretation of a ``true'' parameter and (ii) the notation of generative modelling, in which $t=0$ typically marks the data-generating endpoint.

One of the most desirable  features of GPs is that they remain GPs under Gaussian conditioning with a bounded linear operator \citep{pfortner2022physics}. Let $\mathcal{L}: \mathcal{H} \to \mathbb{R}^n$ be a linear operator acting on the sample paths of $f(\cdot)$ to generate data at locations $\mathbf{X}_n \in \mathcal{X}^n$. Its action on $f(\cdot)$ can be represented as a matrix $\textbf L \in \mathbb{R}^{n\times m}$ acting on a discretisation of $\f(\cdot)$,  $\bm{f}_0$, which also constitutes our inferential target. Now consider the following data model for a noisy realisation $\bm y$,
\begin{equation} \label{eqn:gp_model}
    \bm y = \mathbf{L}\bm{f}_0 + \bm\varepsilon, \quad \bm\varepsilon \sim \mathcal{N}(\bm 0_n, \bm \Gamma),
\end{equation}
where $\bm 0_n$ denotes the $n$-vector of zeros, and where $\bm \Gamma$ denotes the covariance matrix of the noise $\bm \varepsilon$ (e.g. $\bm\Gamma=\sigma^2\mathbf{I}_n$). The GP predictive distribution $\bm{f}_0 \mid \mathcal{D} \sim \mathcal{N}(\mathbf{m}_{*|\bm y}, \mathbf{K}_{**|\bm y})$ has mean and covariance
\begin{align}
\mathbf{m}_{*|\bm y} &=  \mathbf{m}_* + \mathbf{K}_{**} \mathbf{L}^\intercal \left(\mathbf{L} \mathbf{K}_{**} \mathbf{L}^\intercal+ \bm\Gamma\right)^{-1} (\bm y - \mathbf{L} \mathbf{m}_*), \label{eqn:m_post} \\
\mathbf{K}_{**|\bm y} &= \mathbf{K}_{**}- \mathbf{K}_{**}\mathbf{L}^\intercal \left(\mathbf{L} \mathbf{K}_{**} \mathbf{L}^\intercal+ \bm\Gamma\right)^{-1} \mathbf{L} \mathbf{K}_{**}^\intercal, \label{eqn:K_post}
\end{align}
 where $\mathcal{D} = \{\mathbf{X}_n, \bm y\}$. Key applications of Gaussian processes conditioned on linear operators in the Gaussian setting include linear inverse problems \cite{hansen2006linear}; making inference from derivative observations \cite{solak2002derivative,ponte2024inferring}; data assimilation \cite{o1991bayes,briol2019probabilistic}; and problems involving linear constraints \cite{jidling2017linearly}.

\textbf{Techniques for sampling from GPs}. The standard approach is to transform white noise $\bm z\sim\mathcal{N}(\mathbf{0},\mathbf{I}_m)$ via the mean-scale transformation: $\bm f_0 = \mathbf{m}_{*} + \mathbf{K}^{1/2}_{**}\bm z$, where $\mathbf{K}^{1/2}_{**}$ denotes the matrix square root, raising a substantial $\mathcal{O}(m^3)$ factorisation cost. The SPDE approach \citep{whittle1954stationary, lindgren2011explicit} and its Hilbert space extensions \citep{solin2020hilbert} 
specify a Gaussian Markov random field  with sparse precision matrix, whose Cholesky factorization costs roughly $\mathcal{O}(m^{3/2})$; pathwise conditioning \citep{wilson2020efficiently} constructs differentiable predictive sample paths as a prior sample plus a deterministic update; sparse approximations \citep{quinonero2005unifying, hensman2013gaussian, lazaro2010sparse, wilson2015kernel} and computation-aware iterative solvers \citep{cutajar2016preconditioning, wenger2024computation} further reduce the cost of drawing samples from large-scale GP predictive distributions. These methods and approximations are compatible with \textsc{FlowGP} (Section \ref{sec:nonconjugate}) and provide avenues for scalability.

\textbf{Physics-informed Gaussian processes.} Kernel-based approaches encoding physics through the covariance structure are restricted to linear PDEs \citep{ginsbourger2012argumentwise,alvarez2013linear,  ginsbourger2016degeneracy,raissi2017machine, besginow2022constraining, 
steinert2025integration}. Modern approximate methods extend to non-linear PDEs --- \textsc{PHYSS} \citep{hamelijnck2024physics} and \textsc{AUTOIP} \citep{long2022autoip} via variational inference, and via kernel smoothing \cite{chen2021solving}, though the latter yields only point estimates. All of these methods require bespoke derivations tailored to each PDE. Probabilistic ODE solvers \citep{schober2019probabilistic, tronarp2019probabilistic} model solutions as GPs with calibrated discretisation uncertainty and generalise linear PDE solvers \citep{pfortner2022physics}, but extensions to non-linear PDEs remain limited.

\section{Sampling GPs under linear-Gaussian conditioning by solving an ODE}
\label{subsec:flows}

We present a new perspective on GP sampling and  derive an ODE whose dynamics produce the samples from the target Gaussian distribution. Consider the set of vectors $\{\bm f_t\}_{t \in [0,1]}$, each defined by mapping white noise samples  $\bm z$ via the flow map
\begin{align}
      \bm f_t= \Phi(t, \bm z) :=  \mathbf{b}(t) + \mathbf{A}(t)^{1/2}\bm z, \quad t \in [0,1], \label{eq:flowmap}
\end{align}
where $\mathbf{A}(t) := \alpha^2(t) \mathbf{K}_{**} + (1 - \alpha^2(t))\mathbf{I}_m$ and $\mathbf{b}(t) := \alpha(t)\mathbf{m}_*$. We require that $\alpha(t)$ is a monotonically decreasing and differentiable function, with $\alpha(0) = 1$, and $\alpha(1) \rightarrow 0^+$ to ensure that $\bm f_0\sim\mathcal{N}(\mathbf{m}_{*}, \mathbf{K}_{**})$ and $\bm f_1\rightarrow\bm z$. Differentiating the flow map $\Phi(\cdot,\cdot)$ in \eqref{eq:flowmap} with respect to $t$, we obtain the following linear time-varying ODE,

\begin{equation}
\boxed{    \label{eq:GPODE}
    \frac{\mathrm{d}\bm f_t}{\mathrm{d}t} =\bm{v}(\bm f_t, t) \coloneqq -\frac{1}{2}\beta(t)  \mathbf{A}(t)^{-1}\mathbf{b}(t)-\frac{1}{2}\beta(t)\left(\mathbf{I}_m - \mathbf{A}(t)^{-1}\right)\bm f_t, \quad t\in[0,1],
}
\end{equation}
where $v(\bm f_t,t)$ is the velocity field, $\bm f_1 \sim \mathcal{N}(\mathbf{0},\mathbf{I}_m)$ and $\beta(t)  = -2 \frac{\alpha'(t)}{\alpha(t)} = -2\frac{\mathrm{d}\log\alpha(t)}{\mathrm{d} t}$. Across all the results in this paper, we use a simple linear schedule for $\beta(t)$ (see Appendix \ref{appendix:schedule}). The ODE \eqref{eq:GPODE} yields a strategy for sampling from $\bm f_0$: starting from white noise $\bm f_1\sim\mathcal{N}(\mathbf{0},\mathbf{I}_m)$ and numerically integrating the ODE \eqref{eq:GPODE} backwards in time from $t=1$ to $t=0$ produces samples from the target distribution at $t=0$ (see Figure~\ref{fig:linear_ode_evolution}). 

\subsection{An alternative perspective: sampling from Gaussian processes with a diffusion} 

The time-$t$ marginals $p(\bm f_t), t\in [0,1]$ induced by the ODE in \eqref{eq:GPODE} are the same as those of a variance-preserving (VP) SDE \citep{songscore} of the form
\begin{equation} \label{eq:var_preserving}
\mathrm{d}\bm f_t = - \frac{1}{2}\beta(t)\, \bm f_t \,\mathrm{d}t + \sqrt{\beta(t)} \, \mathrm{d} \bm W_t,\quad t\in[0,1],
\end{equation} where $\bm W_t$ is an $m$-dimensional Wiener process. To see this, first observe that the conditional distributions $p(\bm f_t \mid \bm f_0)$ generated by the SDE in \eqref{eq:var_preserving}  are, for all $t \in [0,1]$, the same as those generated by the perturbation model \citep{lipmanflow}
\begin{equation}\label{eq:var_preserving2}
\bm f_t = \alpha(t)\bm f_0 + \sqrt{1-\alpha^2(t)}\bm z,\quad t\in[0,1],
\end{equation}
where $\bm z \sim \mathcal N(\bm 0, \mathbf I_m)$. This can be verified by showing that the conditional expectations and variances of the solution to \eqref{eq:var_preserving} match those of \eqref{eq:var_preserving2} for each $t$. It follows that for a distribution $p(\bm f_0)$, the SDE in \eqref{eq:var_preserving} induces the same marginals
\begin{equation} \label{eq:marginals}
    p(\bm f_t) =\int p(\bm f_t \mid \bm f_0)p(\bm f_0)\, \mathrm{d} \bm f_0, \quad t\in[0,1],
\end{equation}
as \eqref{eq:var_preserving2}.  Second, it is immediate that \eqref{eq:GPODE} generates the same marginals $\bm f_t \sim \mathcal N(\mathbf{b}(t), \mathbf A(t)), t \in [0,1],$ as \eqref{eq:var_preserving2}  when $\bm f_0 \sim \mathcal N(\mathbf{m}_*, \mathbf K_{**})$, and therefore the same as those from the SDE in \eqref{eq:var_preserving}.  We may therefore replace simulation of the SDE with numerical integration of the ODE in \eqref{eq:GPODE} when sampling from these marginals. 

In most applications of diffusion sampling, \eqref{eq:marginals} is approximated with a neural network as a function of $t$, however these marginals are available in closed-form when the distribution of $\bm f_0$ is induced from a GP. Moreover, the conditional distribution $p(\bm f_0 \mid \bm f_t)$ of the SDE is also available; in Appendix \ref{appendix:flow_derivations} we show that for $\bm f_0 \sim \mathcal N(\mathbf{m}_*, \mathbf K_{**})$,
\begin{align} 
     \bm f_0 \mid \bm f_t \sim~ & ~\mathcal{N}\left(\mathbf{m}_* + \alpha(t)\mathbf{K}_{**}\mathbf{A}(t)^{-1}\left(\bm f_t - \alpha(t)\mathbf{m}_*\right), \mathbf{K}_{**} - \alpha^2(t)\mathbf{K}_{**}\mathbf{A}(t)^{-1}\mathbf{K}_{**}\right).\label{eq:joint}
\end{align}
This property will be particularly useful when computing with non-linear guidance in Section~\ref{sec:nonconjugate}.

The above discussion showed that the probabilistic flow \eqref{eq:GPODE} generates margins that are the same as that of the SDE \eqref{eq:var_preserving}. We can also derive the probability flow \eqref{eq:GPODE} directly from the SDE. It is well known  \citep{anderson1982reverse} that the probability-flow ODE associated with the SDE \eqref{eq:var_preserving} is specified by
the score function $s(\bm f_t, t)=\nabla_{\bm f_t}\log p(\bm f_t)$ via
 \begin{equation}
 \label{eq:PFODE}
     \frac{\mathrm{d}\bm f_t}{\mathrm{d}t} = -\frac{1}{2}\beta(t)\left(\bm f_t + s(\bm f_t, t)\right);
 \end{equation}
see \cite{lai2025principles} for a derivation. In our Gaussian setting, the score is simply $s(\bm f_t,t) =  - \mathbf{A}(t)^{-1} \left(\bm f_t - \mathbf{b}(t)\right)$; substituting this into \eqref{eq:PFODE} recovers the linear ODE in \eqref{eq:GPODE}. We use \eqref{eq:PFODE} directly in Sections \ref{subsec:conjugate_conditioning} and \ref{subsec:nonconjugate_conditioning} to implement linear-Gaussian and nonlinear-non-Gaussian guidance, respectively.  The ODE \eqref{eq:PFODE} also admits a natural flow matching interpretation: the interpolant \eqref{eq:var_preserving2} is precisely the affine Gaussian interpolant of stochastic interpolant flow matching \citep{albergo2023building}, making \textsc{FlowGP} an exact closed-form instance of flow matching for Gaussian targets.

\subsection{Sampling under linear Gaussian conditioning as score-based guidance}
\label{subsec:conjugate_conditioning}

Recall the linear Gaussian data model in \eqref{eqn:gp_model} that yields the predictive distribution $\bm{f}_0 \mid \mathcal{D} \sim \mathcal{N}(\mathbf{m}_{*|\bm y}, \mathbf{K}_{**|\bm y})$. As $\bm{f}_0 \mid \mathcal{D}$ remains Gaussian, we may  substitute $\mathbf{b}_{\mid \bm y}(t) \coloneqq \alpha(t)\mathbf{m}_{*\mid\bm y}$ and 
$\mathbf{A}_{\mid \bm y}(t) \coloneqq \alpha^2(t)\mathbf{K}_{**\mid\bm y} + (1-\alpha^2(t))\mathbf{I}_m$ for $\mathbf{b}(t)$ and $\mathbf{A}(t)$ in \eqref{eq:GPODE} to yield the ODE 
\begin{equation}
\label{eq:GPODE2}
    \frac{\mathrm{d}\bm f_t}{\mathrm{d}t} =\bm{v}(\bm f_t, t \mid \mathcal D) \coloneqq -\frac{1}{2}\beta(t)  \mathbf{A}_{|\bm y}(t)^{-1}\mathbf{b}_{|\bm y}(t)-\frac{1}{2}\beta(t)\left(\mathbf{I}_m - \mathbf{A}_{|\bm y}(t)^{-1}\right)\bm f_t, \quad t\in[0,1],
\end{equation}
which can be used to generate samples from the predictive distribution $\bm{f}_0 \mid \mathcal{D} \sim \mathcal{N}(\mathbf{m}_{*|\bm y}, \mathbf{K}_{**|\bm y})$ (see Figure \ref{fig:linear_ode_evolution}), with marginals $\bm f_t \mid \mathcal{D} \sim \mathcal{N}(\mathbf{b}_{\mid \bm y}(t),\mathbf{A}_{\mid \bm y}(t))$.

\begin{figure}

    \includegraphics[width=\textwidth]{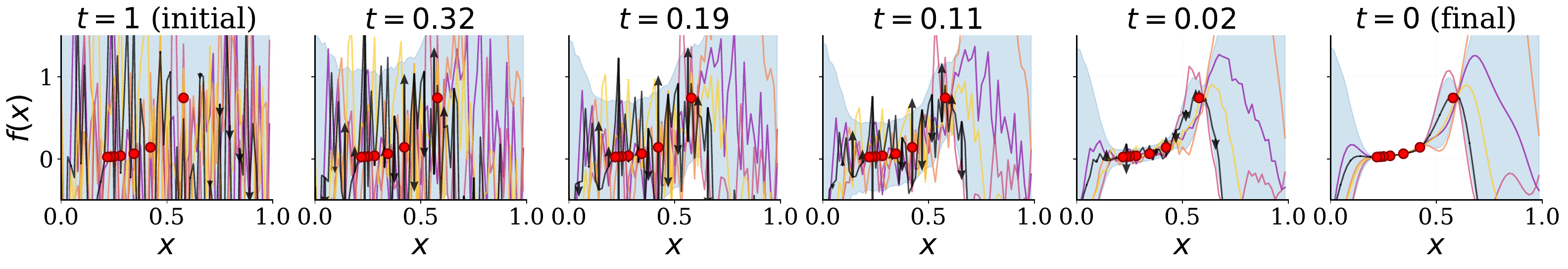}

    \caption{ Generating samples from the GP predictive distribution when conditioning on Gaussian observations (red dots). Sample trajectories (lines), velocity field (arrows), and empirical quantiles of $p(\bm f_t)$ (shading) are shown at different times. The flow interpolates white noise (left, $t = 1$) to samples from the predictive distribution (right, $t = 0$).}
\label{fig:linear_ode_evolution}
\end{figure}

Alternatively, one can start off from the probability flow ODE representation in \eqref{eq:PFODE} given by
\begin{equation}\label{eq:ODE2}
\frac{\mathrm d \bm f_t}{\mathrm d t} = \bm{v}(\bm f_t, t \mid \mathcal{D}) := -\frac{1}{2}\beta(t)\left(
\bm f_t + \nabla_{\bm f_t}\log p(\bm f_t \mid \mathcal D)
\right).
\end{equation}
The conditional score admits the decomposition
\begin{equation} \label{eqn:conjugate_score}
\nabla_{\bm f_t}\log p(\bm f_t \mid \mathcal D)
=
\underbrace{\nabla_{\bm f_t}\log p(\bm f_t)}_\text{unconditional GP}
+
\underbrace{\nabla_{\bm f_t}\log p(\mathcal D \mid \bm f_t),}_\text{linear conditioning/guidance}
\end{equation}
where the first term corresponds to the GP model (Section~\ref{subsec:flows}) %
and the second term encodes the conditioning information. While this recovers standard GP conditioning in a different guise, the decomposition reveals that conditioning can be implemented by augmenting the drift of the probability flow ODE with a \emph{guidance term} $\nabla_{\bm f_t}\log p(\mathcal D \mid \bm f_t)$ \citep{daras2024survey}. Importantly, this framework provides a path forward to relax assumptions of linearity and Gaussianity in our conditioning statements.

\section{\textsc{FlowGP}: sampling GPs under non-linear and non-Gaussian conditioning}
\label{sec:nonconjugate}

We now demonstrate the key utility of our ODE viewpoint of GP sampling: it enables efficient sampling from the GP predictive distribution arising from arbitrary non-linear non-Gaussian likelihoods. Specifically, we target
\begin{equation}
    p(\bm f_0 \mid \mathcal D, \mathcal C) \propto \underbrace{p(\bm f_0)\, p(\mathcal D \mid \bm f_0)}_{\propto \; p(\bm f_0 \mid \mathcal D)} \, p(\mathcal C \mid \bm f_0, \mathcal{D}), \label{eq:GPcond}
\end{equation}
where $\mathcal{D}$ denotes the information that admit a linear closed-form Gaussian update, and $\mathcal{C}$ denotes non-linear and non-Gaussian contributions, assumed conditionally independent given $\bm f_0$. 
We assume the GP prior and $\mathcal{C}$ are broadly compatible; if the prior assigns negligible probability to the region satisfying $\mathcal{C}$, sample quality may degrade and so kernel hyperparameters should be reasonably well-specified before applying \textsc{FlowGP}; see Section~\ref{sec:limitations}.

\subsection{Sampling from the predictive distribution via an ODE}\label{subsec:nonconjugate_conditioning}

Section~\ref{subsec:conjugate_conditioning} outlines two admissible ODE-based methods for sampling from the predictive distribution under linear Gaussian conditioning: differentiating closed-form flow maps and using the probability flow \eqref{eq:ODE2}. In general, no analytically-tractable flow map exists when conditioning includes $\mathcal C$, but the representation \eqref{eq:PFODE} offers a way forward to generate the required samples.

Consider a general conditioning set $\mathcal{C}$ and the sequence of marginal distributions obtained by incrementally noising our target conditional distribution 
\begin{equation*}
    p(\bm f_t \mid \mathcal{D}, \mathcal{C}) = \int p(\bm f_t \mid \bm f_0)p(\bm f_0 \mid \mathcal{D}, \mathcal{C}) \; \mathrm{d} \bm f_0 \quad \text{where} \quad p(\bm f_t \mid \bm f_0) = \mathcal{N}(\alpha(t)\bm f_0, (1-\alpha^2(t))\mathbf{I}_m).
\end{equation*}
By Bayes' rule, the score function of $p(\bm f_t \mid \mathcal{D}, \mathcal{C})$ decomposes as
\begin{align}
\nabla_{\bm f_t}\log p(\bm f_t \mid \mathcal{D}, \mathcal{C})
&= \underbrace{\nabla_{\bm f_t}\log p(\bm f_t \mid \mathcal D)}_{\text{linear update in \eqref{eqn:conjugate_score}}}
+ \underbrace{\nabla_{\bm f_t}\log p(\mathcal C \mid \bm f_t, \mathcal D)}_{\text{non-linear guidance}}.
\label{eq:GPCONDODE}
\end{align}
This score function is the sum of the score of the  linear update and an additional term  encoding the non-linear and non-Gaussian contributions. Incorporating this term into the ODE \eqref{eq:PFODE} yields 
\begin{equation}
\boxed{    \label{eq:NONCONGGPCONDODE}
    \frac{\mathrm{d}\bm f_t}{\mathrm{d}t} =\bm{v}(\bm f_t, t\mid \mathcal{D}) - \frac{1}{2}\beta(t)\nabla_{\bm f_t}\log p(\mathcal{C} \mid \bm f_t, \mathcal{D})
}
\end{equation}
where $\bm{v}(\bm f_t, t\mid \mathcal{D})$ is the vector field \eqref{eq:GPODE2} producing samples from the linear-Gaussian part of the predictive distribution, $p(\bm f_0 \mid \mathcal{D})$.  If $\nabla_{\bm f_t} \log p(\mathcal C \mid \bm f_t, \mathcal{D})$ is available, the flow can therefore be adjusted to generate samples from the full predictive distribution that also involves the nonlinear non-Gaussian components $\mathcal C$. Therefore, the ODE \eqref{eq:NONCONGGPCONDODE} yields a strategy for sampling from the predictive distributions of GPs that are not available in closed form: starting from white noise $\bm f_1\sim\mathcal{N}(\mathbf{0},\mathbf{I}_m)$ and numerically integrating ODE \eqref{eq:NONCONGGPCONDODE} from $t=1$ to $t=0$ produces samples $\bm f_0$ from the target predictive distribution (see Figure \ref{fig:nonlinear_ode_evolution}).

\begin{figure}[t]
    \centering
    
    \includegraphics[trim={0 1cm 0 0}, clip, width=\textwidth]{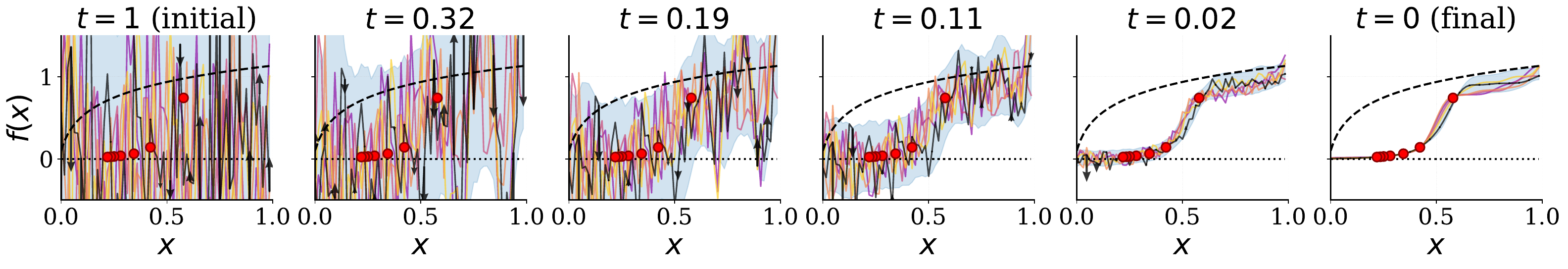}

    \includegraphics[trim={0 0 0 1cm}, clip, width=\textwidth]{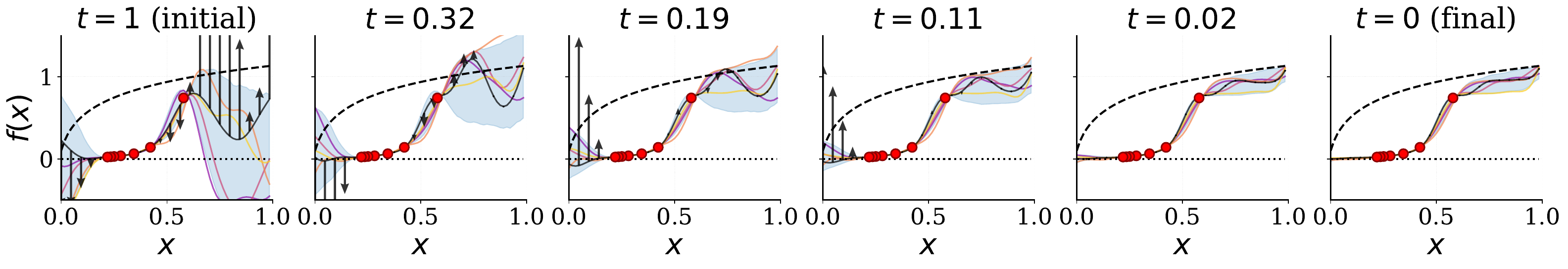}
    
    \caption{Extension of Figure~\ref{fig:linear_ode_evolution} to non-Gaussian and non-linear likelihoods.  \textbf{Top row}: Generating samples from the predictive distribution when conditioning on Gaussian observations (red dots) together with non-linear constraints enforcing boundedness and monotonicity (dashed lines). \textbf{Bottom row}: Generating samples from the same predictive distribution using the whitened formulation, in which the Gaussian dynamics are removed, producing smoother trajectories and avoiding the stiff behaviour present in the original formulation.   
    }
    \label{fig:nonlinear_ode_evolution}
\end{figure}

\subsection{Calculating the guidance term by Monte Carlo}
\label{subsec:MC}

Unlike $\bm{v}(\bm f_t, t\mid \mathcal{D})$, the nonlinear guidance term $\nabla_{\bm f_t}\log p(\mathcal{C} \mid \bm f_t, \mathcal{D})$ appearing in \eqref{eq:NONCONGGPCONDODE} is generally intractable. However, since
\begin{align}\label{eq:guidanceterm}
\nabla_{\bm f_t}\log p(\mathcal C \mid \bm f_t, \mathcal D)
=
\nabla_{\bm f_t}\log \int p(\mathcal C \mid \bm f_0)\, p(\bm f_0 \mid \bm f_t, \mathcal D)\, \mathrm d \bm f_0,
\end{align}
and $p(\bm f_0 \mid  \bm f_t, \mathcal{D})$ is a known Gaussian distribution, we can estimate it using a Monte Carlo approximation that only requires point-wise evaluation of the likelihood $p(\mathcal C\mid \bm f_0)$ and its score $\nabla_{\bm f_0}
\log p(\mathcal C\mid \bm f_0)$.

For a sample of size $S$, 
\begin{equation}
\label{eq:guidance-mc-estimator}
    \nabla_{\bm{f}_t} \log p(\mathcal{C} \mid \bm{f}_t, \mathcal{D})
    \approx \alpha(t)\mathbf{K}_{**|\bm y}\mathbf{A}^{-1}_{|\bm y}(t)\sum_{i=1}^S \bar w^{(i)}s^{(i)}, \quad \textrm{for}\quad\bar w^{(i)} = w^{(i)}/\sum_{r=1}^S w^{(r)}
\end{equation}
where $\bm f_0^{(i)}\sim p(\bm f_0 \mid \bm f_t, \mathcal{D})$, and $w^{(i)} \propto p(\mathcal C\mid \bm f_0^{(i)})$, $s^{(i)}=\nabla_{\bm f_0}\log p(\mathcal{C}\mid \bm f_0^{(i)})$
for $i = 1,\dots,S$. See Appendix~\ref{appendix:guidance_calc}  for derivation and implementation details. If $\nabla_{\bm f_0} \log p(\mathcal C\mid \bm f_0)$ is unavailable, a gradient-free estimate may be obtained via Fisher's identity (Appendix~\ref{appendix:fisher}).

\textbf{Convergence and weight collapse.} The estimator \eqref{eq:guidance-mc-estimator} is a self-normalised importance sampler. It is consistent, with a bias of order $\mathcal{O}(S^{-1})$, which vanishes faster than the $\mathcal{O}(S^{-1/2})$ standard Monte Carlo error under standard moment conditions on the importance weights. In high dimensions, self-normalised estimators can in principle suffer from \emph{weight collapse} onto a single sample. Note that our construction mitigates this risk as the covariance of $\bm f_0|\bm f_t, \mathcal{D}$ remains a substantial fraction of $\mathbf{K}_{**|\bm y}$ for most of the integration interval, providing a broad proposal.
Equivalently, the guidance can be read as a smoothed likelihood: at time $t$ the flow follows a modulated likelihood smoothed by the  conditional $p(\bm f_0|\bm f_t, \mathcal{D})$ resulting in heavy smoothing at large values of $t$ and sharpening to the exact likelihood at terminal time $t=0$. We acknowledge that collapse may nonetheless arise in high-dimensional settings and under tight conditions; developing adaptive remedies is an important direction for future work.

\textbf{Alternative guidance approximations.} In standard diffusion guidance, $p(\bm f_0 \mid \bm f_t, \mathcal D)$ is rarely available in closed form, motivating point-estimate methods such as DPS \citep{chung2023diffusion} and MPGD \citep{manifold}. More sophisticated schemes incorporate higher-order information at additional computational cost \citep{rissanenfree}. Remarkably, in Appendix \ref{appendix:ablate} we show that our simple Monte Carlo estimator outperforms DPS and MPGD even at $S=1$, raising broader questions about the suitability of point-estimate approximations of guidance in diffusion sampling.

\section{Practical and theoretical considerations}

\subsection{A natural parameterisation through whitening}
\label{subsec:whiten}

Discretising the ODE in \eqref{eq:NONCONGGPCONDODE} over $T$ steps is challenging when the covariance is poorly conditioned as the dynamics become stiff, which forces small step sizes. %
In Appendix~\ref{app:theory}, we show that stiffness is governed by the condition number $\kappa(\mathbf K_{**|\bm y})$, which also controls numerical stability in standard GP sampling. In particular, Theorem~\ref{cor:stiffness} establishes that under rescaling and in the regime $\lambda_{\min}(\mathbf{K}_{**|\bm y}) \to 0$, stiffness has two-sided asymptotic equivalence to $\kappa(\mathbf K_{**|\bm y})$, i.e., $\text{stiffness} \asymp \kappa(\mathbf K_{**|\bm y})$. This regime arises, for instance, under increasing data in a fixed domain. Increased stiffness leads to numerical instabilities that we seek to avoid. 

This challenge is resolved by whitening. Define $\hat{\bm f}_0 \coloneqq \mathcal{W}\bm f_0 =  \mathbf{K}_{**|\bm y}^{-1/2}\bm (\bm f_0-\mathbf m_{*|\bm y})$, under which $\hat{\textbf{A}}(t)=\textbf{I}_m$ and $\hat{\bm b}(t)=\bm 0$, so the linear Gaussian velocity term vanishes, leaving only the non-linear guidance contribution
\begin{equation}
\boxed{    \label{eq:whitened}
    \frac{\mathrm{d}\hat{\bm f}_t}{\mathrm{d}t} =
    \cancelto{0}{\bm v(\hat{\bm f}_t, t \mid \mathcal{D})}
    \; - \; \frac{1}{2}\beta(t)\nabla_{\hat{\bm f}_t}\log p(\mathcal{C} \mid \hat{\bm f}_t, \mathcal{D}), 
}
\end{equation}
and the conditional distribution from the underlying whitened SDE becomes $\hat{\bm f}_0 \mid \hat{\bm f}_t, \mathcal D \sim \mathcal{N}(\alpha(t)\hat{\bm f}_t, (1-\alpha^2(t))\mathbf{I}_m)$.
Predictive samples are recovered from the solutions to \eqref{eq:whitened} via $\bm f_0 = \mathcal{W}^{-1}\hat{\bm f}_0$.

\textbf{Why whitening?} The original dynamics of \eqref{eq:NONCONGGPCONDODE} couple two operations: transporting white noise to a correlated Gaussian, and steering samples towards $p(\bm f_0 \mid \mathcal{D}, \mathcal C)$ via non-linear guidance. Whitening handles the first explicitly, so the whitened ODE retains only the irreducible non-Gaussian dynamics. This decomposition is also optimal from a transport perspective: Theorem~\ref{cor:gaussian-action} and Corollary~\ref{cor:isotropy-minimises} show that whitening minimises the upper bound on the squared Wasserstein--2 distance $W_2^2(p(\bm f_0), p(\bm f_1))$ for all schedules $\beta(t)$.

\subsection{Solving the resulting ODE}
\label{subsec:flowgp}

In practice \textsc{FlowGP} incurs two sources of error: (i) the guidance approximation, and (ii) ODE solver discretisation. Theorem~\ref{thm:guidance_stability} shows these decompose cleanly into a guidance error governed by log-likelihood curvature and a discretisation error governed by solver order and step size, which need not compound whenever the likelihood is sufficiently regular. For log-concave likelihoods, including classification, Poisson regression, and convex constraints, guidance error remains bounded or accumulates at most linearly, controllable by increasing $S$ or reducing step size. Non-log-concave settings, such as mixture models or LLM-based likelihoods, admit only a worst-case exponential bound, though this is likely loose in practice: empirical results in Figure~\ref{fig:llm+physc} are well-behaved despite this classification. Understanding precisely when this regularisation suffices is an important direction for future work. A full error analysis is given in Appendix~\ref{sec:numerical error}.

\textbf{Non-differentiable likelihoods.} A conclusion from Theorem~\ref{thm:guidance_stability} is that, in order to
practically solve \textsc{FlowGP}'s ODE under genuinely non-differentiable conditions such as binary or equality constraints, we require a smooth surrogate. Therefore, we propose binary conditions (e.g. inequality constraints) to be wrapped in a probit likelihood with small bandwidth $\nu$, and equality conditions are replaced by a Gaussian likelihood with small variance $\sigma$, each recovering the original constraint in the limit.

\section{Experimental details}
\label{sec:experiments}
 
We now provide details for a number of experimental settings, each chosen to demonstrate a key benefit of \textsc{FlowGP}. First, we establish competitiveness with bespoke constrained regression methods (Section~\ref{sec:monotonic}). Second, we show that our framework occupies a previously unoccupied position in physics-informed GP modelling, handling non-linear PDEs without bespoke kernels at a fraction of the cost (Section~\ref{sec:physics}). We then demonstrate that encoding structural knowledge into the GP predictive distribution via \textsc{FlowGP} improves a down-stream decision-making in Bayesian optimisation (Section~\ref{sec:BO}). Next, we consider a controlled setting where the likelihood function is sharp. Finally, we explain how we used \textsc{FlowGP} in tandem with large language models to condition on natural language --- a capability beyond the reach of any existing GP method (Section~\ref{sec:LLM}).

\textbf{Implementation details.}  We employ a linear diffusion schedule $\beta(t)=10^{-5} + 10t$ and solve the resulting ODE via a simple Euler scheme for $1000$ time steps. The SNR-uniform time discretisation of \cite{nichol2021improved} is used throughout (see Appendix \ref{appendix:snr}); this concentrates steps near $t=0$ where Theorem~\ref{cor:stiffness} identifies stiffness to be maximal, providing a principled justification for this choice in the GP setting. Supported by our ablation in Appendix \ref{appendix:ablate}, we use only $S=5$ Monte Carlo samples for all guidance approximations, unless stated otherwise. Unless otherwise stated, kernel hyperparameters are estimated by maximising the marginal likelihood on the linear-Gaussian component $p(\bm f_0 \mid \mathcal{D})$ before applying our ODE sampler to the full predictive distribution $p(\bm f_0 \mid \mathcal{D}, \mathcal{C})$. Performance is evaluated on held-out test data using root mean squared error (RMSE) and negative log predictive density (NLPD), estimated empirically from 100 samples. Reported timings  include both hyperparameter fitting and sampling, averaged over 100 independent runs. All experiments used a Nvidia A4500 GPU workstation except the one-dimensional monotonic experiment which used an Apple M2 Pro processor with 32GB RAM to allow fair comparisons with competitor approaches. 
Code is available at \href{https://github.com/FlowInference/flow-inference}{\texttt{github.com/FlowInference/flow-inference}}, with all additional implementation details provided in Appendices \ref{app:monotonic}-\ref{app:llm}.

\subsection{Monotone and bounded regression: sanity check on a well-studied problem}
\label{sec:monotonic}

Figure~\ref{fig:nonlinear_ode_evolution} validated \textsc{FlowGP} on a well-studied benchmark, demonstrating that our general framework can match the performance of custom problem-specific approaches. We reproduce the case study of \cite{agrell2019gaussian}: recovering $f(x) = \frac{1}{3}\left[\tan^{-1}(20x-10) - \tan^{-1}(-10)\right]$ from seven observations under monotonicity and the bound $0<f(x)<\frac{1}{3}\log(30x+1) + 0.1=u(x)$. Both constraints are handled as inequality conditions via the probit relaxation of Section~\ref{subsec:flowgp}: monotonicity is encoded as a lower bound on finite differences of $\bm f_0$, and boundedness as pointwise upper and lower bounds. Figure~\ref{fig:monotone_full} in Appendix~\ref{app:monotonic} shows that \textsc{FlowGP} closely matches or outperforms bespoke methods from the statistics literature \citep{lin2014bayesian, astfalck2018posterior} that are specifically engineered for this constraint structure, whilst avoiding the computational overhead of virtual observation approaches \citep{agrell2019gaussian, wang2014modeling, riihimaki2010gaussian, da2012gaussian}.

\subsection{Physics-obeying GPs: non-linear constraints without bespoke kernels or inference}
\label{sec:physics}
We demonstrate \textsc{FlowGP} on three physics benchmarks --- the non-linear pendulum, Allen--Cahn, and Burgers' equations --- each handled by specifying only the likelihood $p(\mathcal{C} \mid \bm f_0)$ via ODE and boundary condition residuals. Figure~\ref{fig:llm+physc} shows physically consistent samples across all three, and Table~\ref{tab:physics} confirms competitive accuracy at orders-of-magnitude lower cost with improved uncertainty quantification over the first two benchmarks. See Appendix \ref{app:burgers} for qualitatively similar results on Burger's equations and an additional comparison with non-probabilistic kernel method of \citep{chen2021solving}. Physics constraints are enforced via finite differences at the discretisation grid; finer grids or spectral differentiation are straightforward substitutes requiring no changes to the sampling algorithm. We note that hyperparameter fitting procedures differ between methods: \textsc{FlowGP} fits by maximising the marginal likelihood, whilst \textsc{PHYSS} and \textsc{AUTOIP} optimise within their respective inference frameworks. This may partially confound accuracy comparisons, e.g. the PDE constraint may strongly govern smoothness and lengthscale.

\begin{table}[t]
\centering
{\setlength{\tabcolsep}{5pt}
\renewcommand{\arraystretch}{0.8}
\begin{tabular}{@{}c c c@{}}
\begin{tabular}{@{}p{1.2cm}llll@{}}
\toprule
\textbf{Model} & \textbf{RMSE} & \textbf{NLPD} & \textbf{Time} \\
\midrule
$\textrm{PHYSS}$  & \textbf{0.05} & -0.38 & $1\!\cdot\!10^{2}$ \\
$\textrm{AUTOIP}$ & \textbf{0.05} & -0.08 & $1\!\cdot\!10^{2}$ \\
\midrule
$\textrm{FLOW}$         & 
\textbf{0.05 (0.00)} & -0.68 (0.10) & $\mathbf{4\!\cdot\!10^{0}}$ \\
\midrule
$\textrm{FLOW}_{\textrm{W}}$ & \textbf{0.05 (0.00)}   & \textbf{-1.05 (0.15)} & $\mathbf{4\!\cdot\!10^{0}}$ \\
\bottomrule
\end{tabular}
&
\begin{tabular}{@{}p{1.2cm}llll@{}}
\toprule
\textbf{Model} & \textbf{RMSE} & \textbf{NLPD} & \textbf{Time} \\
\midrule
$\textrm{PHYSS}$ & 0.17 & 1.69 & $5\!\cdot\!10^{3}$ \\
$\textrm{AUTOIP}$             & 0.17 & -0.29 & $1\!\cdot\!10^{4}$ \\
\midrule
$\textrm{FLOW}$               & 0.38 (1.39) & 3.20 (0.05) & $\mathbf{4\!\cdot\!10^{0}}$ \\
\midrule
$\textrm{FLOW}_{\textrm{W}}$  & \textbf{0.13(0.00)} & \textbf{-0.83(0.02)} & $\mathbf{4\!\cdot\!10^{0}}$ \\
\bottomrule
\end{tabular}
\end{tabular}
}

\vspace{0.5em}

\caption{
\textsc{FlowGP} matches existing physics-obeying approaches in accuracy at orders-of-magnitude 
lower cost across physics benchmarks, with whitening ($_W$) consistently improving uncertainty 
quantification.
\textbf{Left:} Non-linear pendulum ($25\times$ lower cost). 
\textbf{Right:} Allen--Cahn equation ($2{,}500\times$ lower cost). 
Timings in seconds; results to 2 d.p.\ with one standard deviation (in parentheses)
and best-performing score in \textbf{bold}, each based on 100 independent runs, each producing 100 samples in a single batched computation. Baseline results are taken from the respective 
publications, which did not report confidence intervals but used comparable hardware 
(Nvidia Titan RTX vs. our Nvidia A4500).
}
\label{tab:physics}
\end{table}

\subsection{Bayesian Optimisation: \textsc{FlowGP} supports well-calibrated decision making} \label{sec:BO}
GPs are widely used for decision-making under uncertainty, with Bayesian optimisation (BO) being the most prominent example; see \cite{garnett2023bayesian} for a discussion. Certain advanced BO settings benefit from structural constraints: \cite{wang2025bayesian} show that for Bayesian Optimisation with Preference Exploration (BOPE), replacing the standard preference-learning GP (i.e.\ a probit-likelihood GP that learns via pairwise comparison data using a Laplace approximation) \citep{chu2005preference} with a carefully engineered Monotonic Neural Network Ensemble (MoNNE) yields state-of-the-art performance. We consider exactly this setting, but instead of a bespoke architecture, we simply apply \textsc{FlowGP} with monotonicity as a constraint on the standard preference-learning GP. As can be seen in \Cref{fig:fullbo} of Appendix \ref{appendix:BO}, we found \textsc{FlowGP} to significantly improve GP performance across all six benchmark functions found in \citep{wang2025bayesian}, often matching MoNNE performance. \textsc{FlowGP} thus provides a principled and general way to incorporate structural knowledge into BO, where well-calibrated uncertainty is critical and bespoke architectures are typically required.

\subsection{Analytical test: sensitivity to sharp likelihoods}
Our penultimate experiment tests \textsc{FlowGP} where the ground truth is known: conditioning a single draw from a 1D RBF GP prior on $n$  noiseless observations at random locations. We run \textsc{FlowGP} (1000 ODE steps, $S=5$) with point-wise Gaussian-likelihood guidance at noise scale $\sigma$, and report the 2-Wasserstein distance between a Gaussian fitted to 1000 guided samples and the true posterior (Table~\ref{tab:stress}; mean and 95\% CI over 10 runs). Accuracy is not monotonic in $\sigma$: the distance first falls as the guidance becomes more informative, then rises for $\sigma \leq 10^{-6}$. This matches Theorem~\ref{thm:guidance_stability}, whose discretisation constant grows with the likelihood curvature $1/\sigma^2$, so at a fixed step size discretisation error eventually dominates. At these scales the noise variance is also below the nugget typically used to stabilise GP computations, so some numerical difficulty is expected.

\begin{table}[h]
\centering
\caption{$W_2$ distance between FlowGP samples and the true GP posterior (mean $\pm$ 95\% CI over 10 runs).}
\label{tab:stress}
\begin{tabular}{cccc}
\toprule
$\sigma$ & $M=2$ & $M=10$ & $M=40$ \\
\midrule
$10^{-2}$ & $0.16 \pm 0.01$ & $0.57 \pm 0.01$ & $0.27 \pm 0.01$ \\
$10^{-3}$ & $0.17 \pm 0.01$ & $0.28 \pm 0.01$ & $0.05 \pm 0.01$ \\
$10^{-4}$ & $0.18 \pm 0.02$ & $0.15 \pm 0.01$ & $0.05 \pm 0.01$ \\
$10^{-5}$ & $0.16 \pm 0.02$ & $0.12 \pm 0.01$ & $0.05 \pm 0.01$ \\
$10^{-6}$ & $0.47 \pm 0.01$ & $0.63 \pm 0.01$ & $0.67 \pm 0.01$ \\
$10^{-7}$ & $0.49 \pm 0.01$ & $0.64 \pm 0.01$ & $0.68 \pm 0.01$ \\
\bottomrule
\end{tabular}
\end{table}

\subsection{LLM-guided Gaussian processes: conditioning on anything, including text}
\label{sec:LLM}

Our final  experiment demonstrates conditioning on natural language, serving as a proof of concept highlighting a new frontier enabled by our framework. We demonstrate conditioning on natural language via a product-of-experts construction, defining the target as $\pi(\bm f_0 \mid \mathcal{D}, \mathcal{C})\propto {p(\bm f_0\mid \mathcal{D})}{q(\bm f_0\mid \mathcal{C})}$, where $q(\bm f_0 \mid \mathcal{C})$ is an LLM-derived density (Qwen3.5 \citep{team2026qwen3}) over functions conditioned on text, as used previously to guide neural diffusion processes \cite{biggs2026llm}. 
The score decomposes as ${\nabla_{\bm f_t}\log \pi(\bm f_t \mid \mathcal{D}, \mathcal{C})} = {\nabla_{\bm f_t}\log p(\bm f_t\mid\mathcal{D})} + {\nabla_{\bm f_t}\log \tilde q(\bm f_t\mid\mathcal{C})}$, where $\tilde q (\bm f_t\mid\mathcal{C}) :=\int q(\bm f_0\mid\mathcal{C})p(\bm f_t\mid\bm f_0)\mathrm{d}\bm f_0$ is the Gaussian-smoothed LLM density. Here, the guidance term is structurally identical to Section~\ref{subsec:MC} and our Monte Carlo approximation applies directly. We adopt the product-of-experts to avoid backpropagation through the LLM, though direct conditioning is feasible given sufficient compute. Figure~\ref{fig:llm+physc} shows that  \textsc{FlowGP}'s samples are simultaneously prior-consistent and faithful to semantic descriptions. It takes less than 2 seconds to generate $100$ samples for each prompt. Full text prompts and GP priors are given in  Appendix \ref{app:llm}. We anticipate that as LLMs become richer sources of ``qualitative data'', \textsc{FlowGP} provides a principled mechanism for incorporating that knowledge into GP models.

\section{Limitations and future work}
\label{sec:limitations}
We have established a plug-and-play GP inference scheme supporting arbitrary 
likelihoods by leveraging an equivalence between Gaussian processes 
and linear diffusion models, with predictive sampling implemented as a guided ODE. Three key limitations of \textsc{FlowGP} remain. 

\textbf{Scalability and discretisation.} Cubic scaling in $m$ makes large-scale problems intractable, and all theoretical results are stated for the finite-dimensional discretisation rather than the function-space limit; a rigorous infinite-dimensional treatment, together with integration with sparse GP approximations and spectral differentiation schemes, remains open.  A version of \textsc{FlowGP} that uses a random Fourier feature decomposition instead of a fine discretization of $\mathcal{X}$ can be found in \citep{cowperthwaite2026feature}. It has also been shown that neural score-approximations trained over GP priors can largely ameliorate such computation cost \citep{zammit2026spatial}.

\textbf{Challenging likelihoods.} Theoretical guarantees are strong only for log-concave likelihoods; non-log-concave settings admit only worst-case exponential bounds. Nevertheless, \textsc{FlowGP} performs well empirically across all settings, including those where theory is weakest. We conjecture that Gaussian smoothing induced by the noising process regularises non-log-concave likelihoods in practice. Importance weight collapse remains a plausible concern in high dimensions. Understanding the precise conditions under which these failure modes arise, and developing adaptive remedies is an important direction for future work. Alternative parametrisations of the diffusion model may also yield advantageous performance \citep{sharrock2026latentflow}.

\textbf{Hyperparameter misspecification.} Kernel hyperparameters are estimated from $p(\bm f_0 \mid \mathcal{D})$ alone, introducing a systematic bias when $\mathcal{C}$ is highly informative. A natural remedy is to optimise hyperparameters jointly over  $p(\bm f_0 \mid \mathcal{D}, \mathcal{C})$, perhaps by backpropagating through the ODE solver, though this introduces additional computational cost. Adaptive remedies for prior--constraint incompatibility, such as re-initialising hyperparameters when the guidance term dominates, are also natural extensions.

\newpage
\begin{ack}
    HM and CN were supported by Lancaster University’s Mathematics for AI in
Real-world Systems E3 grant. HM was also supported by a UKRI-funded Future Leaders Fellowship, and Schmidt Sciences.
CD \& PH acknowledge financial support by the European Research Council through ERC CoG Action 101123955 ANUBIS; the DFG Cluster of Excellence ``Machine Learning - New Perspectives for Science'', EXC 2064/2; the DFG SPP 2298 (Project HE 7114/6-1), and the Carl Zeiss Foundation project ``Certification and Foundations of Safe Machine Learning Systems in Healthcare'', as well as funds from the Ministry of Science, Research and Arts of the State of Baden-Württemberg. CN was supported by EPSRC grants EP/V022636/1 and EP/Y028783/1. The research was partially enabled by the ARC ITRH for Transforming Energy Infrastructure through Digital Engineering (TIDE; Grant No. IH200100009). TC was supported by the UKRI Centre for Doctoral Training in Applications of Artificial Intelligence to the study of Environmental Risks [EP/S022961/1].
\end{ack}

\bibliography{references}
\bibliographystyle{unsrt}

\newpage

\appendix
\onecolumn

\section{Derivation of our flow's marginal and joint distributions}
\label{appendix:flow_derivations}

Under the prior $\bm f_0 \sim \mathcal{N}(\mathbf{m}_*, \textbf{K}_{**} )$ and the corruption model \eqref{eq:marginals}, the pair $(\bm f_0,\bm f_t)$ is jointly Gaussian. Writing
\[
\bm f_t
=
\alpha(t)\bm f_0+\sqrt{1-\alpha^2(t)}\bm z,
\qquad
\bm z\sim \mathcal N(\bm 0,\mathbf{I}_m),
\qquad
\bm z\perp \bm f_0,
\]
we obtain
\[
\E[\bm f_t]=\alpha(t)\mathbf{m}_*,
\qquad
\textrm{Cov}(\bm f_0,\bm f_t)=\alpha(t)\K_{**},
\]
and
\begin{equation}
\label{eq:A_definition_notes}
\textrm{Cov}(\bm f_t)=\mathbf{A}(t):=\alpha^2(t)\K_{**}+(1-\alpha^2(t))\mathbf{I}_m.
\end{equation}

Hence, the joint distribution is
\[
\begin{bmatrix}
\bm f_0\\
\bm f_t
\end{bmatrix}
\sim
\mathcal N\left(
\begin{bmatrix}
\mathbf{m}_*\\
\alpha(t)\mathbf{m}_*
\end{bmatrix},
\begin{bmatrix}
\K_{**} & \alpha(t)\K_{**}\\
\alpha(t)\K_{**} & \mathbf{A}(t)
\end{bmatrix}
\right).
\]

Standard Gaussian conditioning then gives
\begin{equation}
\label{eq:bridge_distribution_notes}
\bm f_0 \mid \bm f_t
\sim
\mathcal N\bigl(\bm\mu_{0\mid t},\mathbf\Sigma_{0\mid t}\bigr),
\end{equation}
with
\begin{equation}
\label{eq:bridge_mean_notes}
\bm\mu_{0\mid t}
=
\mathbf{m}_*
+
\alpha(t)\K_{**}\mathbf{A}(t)^{-1}
\bigl(\bm f_t-\alpha(t)\mathbf{m}_*\bigr),
\end{equation}
and
\begin{equation}
\label{eq:bridge_cov_notes}
\mathbf\Sigma_{0\mid t}
=
\K_{**}
-
\alpha^2(t)\K_{**}\mathbf{A}(t)^{-1}\K_{**}.
\end{equation}

In particular (as will be useful for Appendix \ref{appendix:fisher}),

\begin{equation}
\label{eq:bridge_mean_expectation_notes}
\E[\bm f_0\mid \bm f_t]
=
\mathbf{m}_*
+
\alpha(t)\K_{**}\mathbf{A}(t)^{-1}
\bigl(\bm f_t-\alpha(t)\mathbf{m}_*\bigr).
\end{equation}

Note that conditioning on $\mathcal{D}$ simply replaces $\mathcal{N}(\mathbf{m}_*,\mathbf{K}_{**})$ by $\mathcal{N}(\mathbf{m}_{*|\bm{y}},\mathbf{K}_{**|\bm{y}})$, so $\bm{\mu}_{0\mid t}$ and $\mathbf{\Sigma}_{0\mid t}$ are the same expressions with $\mathbf{m}_{*|\bm{y}}$, $\mathbf{K}_{**|\bm{y}}$ and $\mathbf{A}_{|\bm{y}}(t)$ in place of $\mathbf{m}_*, \mathbf{K}_{**}$ and $\mathbf{A}(t)$.

\textbf{Under whitening.} In whitened space, the mean is $\mathbf{0}_m$ and the covariance matrix is $\mathbf{I}_m$, so we have the greatly simplified conditional distribution 

\begin{equation}
\label{eq:whitened_bridge_distribution_notes}
\hat{\bm f}_0 \mid \hat{\bm f}_t
\sim
\mathcal N\bigl(\alpha(t)\hat{\bm f}_t, (1-\alpha^2(t))\textbf{I}_m\bigr).
\end{equation}

\section{Approximating the guidance term}
\label{appendix:guidance_calc}
 For each ODE evaluation, we draw $S$ samples $\bm f_0^{(i)}\sim p(\bm f_0\mid\bm f_t, \mathcal{D}), i = 1,\dots,S,$ to obtain $p(\mathcal{C} \mid \bm f_t, \mathcal{D})\approx \frac{1}{S}\sum_i p(\mathcal{C} \mid \bm f_0^{(i)})$. We first take the gradient with respect to $\bm f_t$,
\begin{align}
    \nabla_{\bm{f}_t} \log p(\mathcal{C} \mid \bm{f}_t, \mathcal{D})
    \approx& \nabla_{\bm f_t} \log \frac{1}{S}\sum_{i=1}^S p(\mathcal{C} \mid \bm f_0^{(i)})=  \sum_{i=1}^S\frac{ \nabla_{\bm f_t}p(\mathcal{C}\mid\bm f_0^{(i)})}{ \sum_{r} p(\mathcal{C}|\bm f_0^{(r)})}, \label{eq:overall}
\end{align}
and then apply the chain rule to each term $ \nabla_{\bm f_t}p(\mathcal{C}\mid\bm f_0^{(i)})$,
\begin{align}
    \nabla_{\bm f_t}p(\mathcal{C}\mid\bm f_0^{(i)}) = \left[\frac{\partial \bm f_0^{(i)}}{\partial \bm f_t}\right]^T\nabla_{\bm f_0}p(\mathcal{C}\mid\bm f_0^{(i)}).
\end{align}
Rewriting using the log-derivative trick, we obtain
\begin{align}
    \nabla_{\bm f_t}p(\mathcal{C}\mid\bm f_0^{(i)}) = p(\mathcal{C}\mid\bm f_0^{(i)})\left[\frac{\partial \bm f_0^{(i)}}{\partial \bm f_t}\right]^T\nabla_{\bm f_0}\log p(\mathcal{C}\mid\bm f_0^{(i)}). \label{eq:eachbit}
\end{align}
Substituting \eqref{eq:eachbit} into \eqref{eq:overall} yields 
\begin{align}
    \nabla_{\bm{f}_t} \log p(\mathcal{C} \mid \bm{f}_t, \mathcal{D})
    \approx \sum_{i=1}^S\frac{p(\mathcal{C}\mid\bm f_0^{(i)})\left[\frac{\partial \bm f_0^{(i)}}{\partial \bm f_t}\right]^T \nabla_{\bm f_0}\log p(\mathcal{C}\mid\bm f_0^{(i)})}{ \sum_{r} p(\mathcal{C}\mid\bm f_0^{(r)})}.
\end{align}
Introducing normalised importance weights $\bar w^{(i)} = p(\mathcal{C}\mid\bm f_0^{(i)}) / \sum_r p(\mathcal{C}\mid\bm f_0^{(r)})$ and noting that, under our VP corruption kernel, the Jacobian is the symmetric matrix $\frac{\partial \bm f_0^{(i)}}{\partial \bm f_t} = \alpha(t)\textbf{K}_{**}\mathbf{A}^{-1}(t)$, leads to our proposed guidance term approximation:
\begin{align}
\boxed{
    \nabla_{\bm{f}_t} \log p(\mathcal{C} \mid \bm{f}_t, \mathcal{D})
    \approx \alpha(t)\textbf{K}_{**|\bm y}\mathbf{A}^{-1}_{|\bm y}(t)\left[\sum_{i=1}^S \bar w^{(i)}\nabla_{\bm f_0}\log p(\mathcal{C}\mid\bm f_0^{(i)})\right].
    \label{eq:finalform}
    }
\end{align}
See Appendix \ref{app:practical_guidance} for practical implementation details.

\section{An alternative guidance approximation using Fisher's Identity}
\label{appendix:fisher}
For a gradient-free estimate of the score we may consider Fisher's Identity. First, consider the predictive distribution of $\bm f_t$. The derivative of this quantity is
\[
\nabla_{\bm f_t} p(\bm f_t\mid \mathcal{D}, \mathcal C)
=
\int \nabla_{\bm f_t} p(\bm f_t\mid \bm f_0)p(\bm f_0\mid \mathcal{D}, \mathcal C)\dd \bm f_0.
\]
Using
\[
\nabla_{\bm f_t} p(\bm f_t\mid \bm f_0)
=
p(\bm f_t\mid \bm f_0)\nabla_{\bm f_t}\log p(\bm f_t\mid \bm f_0),
\]
we obtain
\[
\nabla_{\bm f_t} p(\bm f_t\mid \mathcal{D}, \mathcal C)
=
\int p(\bm f_t\mid \bm f_0)\nabla_{\bm f_t}\log p(\bm f_t\mid \bm f_0)p(\bm f_0\mid \mathcal{D}, \mathcal C)\dd \bm f_0.
\]
Dividing by $p(\bm f_t\mid \mathcal{D}, \mathcal C)$ on both sides and applying Bayes' Theorem gives
\begin{equation}
\label{eq:fisher_identity_notes}
\nabla_{\bm f_t}\log p(\bm f_t\mid \mathcal{D}, \mathcal C)
=
\E\left[
\nabla_{\bm f_t}\log p(\bm f_t\mid \bm f_0)
\middle|
\bm f_t,\mathcal{D}, \mathcal C
\right].
\end{equation}

Under a variance preserving Gaussian corruption kernel we have that,
\begin{equation}
\label{eq:gaussian_kernel_score_notes}
\nabla_{\bm f_t}\log p(\bm f_t\mid \bm f_0)
=
-\frac{1}{1-\alpha^2(t)}\bigl(\bm f_t-\alpha(t)\bm f_0\bigr).
\end{equation}
Substituting \eqref{eq:gaussian_kernel_score_notes} into \eqref{eq:fisher_identity_notes} therefore gives
\begin{equation}
\label{eq:conditional_score_bridge_notes}
\nabla_{\bm f_t}\log p(\bm f_t\mid \mathcal{D}, \mathcal C)
=
-\frac{1}{1-\alpha^2(t)}
\left(
\bm f_t-\alpha(t)\E[\bm f_0\mid \bm f_t,\mathcal{D}, \mathcal C]
\right).
\end{equation}

Similarly, in the case where we only condition on $\mathcal{D}$,
\begin{equation}
\label{eq:unconditional_score_bridge_notes}
\nabla_{\bm f_t}\log p(\bm f_t\mid \mathcal{D})
=
-\frac{1}{1-\alpha^2(t)}
\left(
\bm f_t-\alpha(t)\E[\bm f_0\mid \bm f_t, \mathcal{D}]
\right).
\end{equation}

From \eqref{eq:GPCONDODE}, we can subtract \eqref{eq:unconditional_score_bridge_notes} from \eqref{eq:conditional_score_bridge_notes} to get the guidance identity:
\begin{equation}
\label{eq:guidance_identity_notes}
\boxed{
\nabla_{\bm f_t}\log p(\mathcal C\mid \bm f_t, \mathcal{D})
=
\frac{\alpha(t)}{1-\alpha^2(t)}
\left(
\E[\bm f_0\mid \bm f_t, \mathcal{D}, \mathcal C]
-
\E[\bm f_0\mid \bm f_t, \mathcal{D}]
\right).
}
\end{equation}
The guidance is thus a scaled difference between the expected  latent state $\bm f_0$  under full conditioning versus that conditional just on $\mathcal{D}$. Crucially, this decomposition separates tractable and intractable components, with the second term available in closed-form due to our knowledge of the conditional distribution under our noising process which, analogous to \eqref{eq:joint} , is given by 
\begin{align} 
     \bm f_0 \mid \bm f_t, \mathcal{D} \sim~ & ~\mathcal{N}\left(\mathbf{m}_{*|\bm{y}} + \alpha(t)\mathbf{K}_{**|\bm{y}}\mathbf{A}_{|\bm{y}}(t)^{-1}\left(\bm f_t - \alpha(t)\mathbf{m}_{*|\bm{y}}\right),\right.  \nonumber \\ 
     &\left.\qquad \mathbf{K}_{**|\bm{y}} - \alpha^2(t)\mathbf{K}_{**|\bm{y}}\mathbf{A}(t)^{-1}\mathbf{K}_{**|\bm{y}}\right).\label{eq:jointconditional}
\end{align}

\textbf{A Monte Carlo approximation of} $\mathbb{E}[\bm f_0\mid \bm f_t, \mathcal{D}, \mathcal{C}]$: This conditional expectation is intractable for arbitrary non-linear or simulator-based likelihoods. However, we can exploit the fact that $p(\bm f_0 \mid \bm f_t,\mathcal{D})$ in \eqref{eq:jointconditional} is Gaussian, enabling a simple Monte Carlo approximation that requires the weaker assumption that our likelihood $p(\mathcal C\mid \bm f_0)$ can be evaluated point-wise. Specifically, for each ODE evaluation, we sample from the Gaussian distribution $\bm f_0^{(i)}\sim p(\bm f_0\mid\bm f_t,\mathcal{D})$ for $i=1,..., S$ , compute importance weights $w^{(i)} \propto p(\mathcal C\mid \bm f_0^{(i)})$, and form the self-normalised approximation
\begin{equation}
\label{eq:MC}
\E[\bm f_0\mid \bm f_t,\mathcal{D}, \mathcal C]
\approx
\sum_{i=1}^S \bar w^{(i)} \bm f_0^{(i)},
\qquad
\bar w^{(i)} = \frac{w^{(i)}}{\sum_{r} w^{(r)}}.
\end{equation}

This estimator has high variance when the conditioning event is sharp or the dimensionality is large, likely due to it computing a difference of conditional expectations rather than exploiting local shape information through a gradient. Indeed, in Appendix \ref{appendix:ablate} we provide empirical evidence that this estimator requires orders of magnitude more samples than our gradient-based Monte Carlo estimator in Appendix~\ref{appendix:guidance_calc}.

\section{Numerical Stability} \label{app:theory}

We examine two ways to assess the discretisation error from the ODE solver. The first is via \textit{stiffness}, which provides an notion of maximal instantaneous change in the solution of the linear system. This is an important notion for solvers with fixed discretisation. We will show that this is linked to the condition number of a GP. The next is via \textit{transport}, which provides a measure of change from the initial to final states. This provides a richer notion of solver error as it characterises the whole path.

\subsection{Stiffness}

For prior and linear conditional sampling, the numerical stability of the ODE solution can be characterised through the \emph{stiffness} of the underlying dynamical system. Loosely, stiffness provides a measure of the step size required for stable numerical integration: stiffer systems require relatively small time steps, and vice-versa. Consider a linear ODE of the form $\frac{d \bm{f}_t}{dt} = -\textbf{D}(t)\, \bm{f}_t$. For general $\textbf{D}(t)$, a common definition of stiffness is
\begin{equation} \label{eqn:stiffness}
 \mathrm{stiffness} = \max_{t \in [0,1]} \left\{ \frac{\sigma_{\max}(\textbf{D}(t))}{\sigma_{\min}(\textbf{D}(t))}\right\}
\end{equation}
where $\sigma_{\max}(\textbf{D}(t))$ and $\sigma_{\min}(\textbf{D}(t))$ are the maximum and minimum singular values of $\textbf{D}(t)$, respectively. In the traditional setting of diffusion models, the matrix $\textbf{D}(t)$ depends on the unknown score function and is therefore not available in closed form. 
In contrast, in our GP setting where $p(\bm f_0)$ is specified by a GP prior, \eqref{eq:GPODE} admits an explicit linear representation. Below, we provide Theorem~\ref{cor:stiffness} that applies uniformly to both prior sampling with $\textbf K_{**}$ (from \ref{eq:GPODE}) and linear-conditional sampling with $\textbf K_{** \mid \bm y}$ (from \ref{eq:GPODE2}). In what follows, $\textbf K$ denotes either case, as appropriate.

\begin{theorem} \label{cor:stiffness}
As $n\to\infty$, define a sequence $\{\K_n\}_{n \in \mathbb{Z}^+}$ with corresponding sequence of condition numbers $\kappa(\K_n)=\lambda_{n,\max}/\lambda_{n,\min}$. Assume the regime where maximum eigenvalues $\lambda_{n, \max}< 1$ (attainable by rescaling $\K_n$; see the proof) are bounded in $n$ (i.e. $\sup_n \lambda_{n,\max} < 1$) and the sequence of minimum eigenvalues $\lambda_{n,\min}\to 0$. Define the corresponding stiffness sequence $S_n$ induced by $\K_n$ in \eqref{eq:GPODE}. Then
\[
S_n
=
\kappa(\mathbf K_n)\,
\frac{1-\lambda_{n,\min}}{1-\lambda_{n,\max}},
\]
and consequently $S_n \asymp \kappa(\K_n)$ has double-sided asymptotic equivalence as $n\to\infty$.
\end{theorem}

\begin{proof}
Since $\kappa(\K_n)$ is invariant to positive scalar rescaling, we run the ODE on the rescaled problem $\K_n/\sigma_n^2$ with $\sigma_n^2 = c\,\lambda_{n,\max}$ for a fixed $c>1$, and recover samples of $\bm f_0$ by multiplying by $\sigma_n$; this gives $\sup_n \lambda_{n,\max} = 1/c < 1$, and below $\K_n$ denotes the rescaled matrix. The GP probability-flow ODE has negative Jacobian
\begin{equation} \label{eqn:jacobian}
    \textbf{D}(t) = \frac{1}{2} \beta(t)(\textbf{I}_m - \textbf{A}(t)^{-1})
\end{equation}
where $\textbf{A}(t) = \alpha^2(t) \textbf{K} + (1 - \alpha^2(t))\textbf{I}_n$.
Denote here the ordered set of eigenvalues of $\textbf{A}(t)$ as $\lambda_1(t) \geq \cdots \geq \lambda_n(t) > 0$. From \eqref{eqn:jacobian}, the eigenvalues of $\textbf{D}(t)$ are given as
\[
\mu_i(t) = \frac{1}{2}\beta(t)(1 - \lambda_i^{-1}(t)),
\]
and since $\textbf{D}(t)$ is symmetric, its singular values are
\[
\sigma_i(t) = |\mu_i(t)| = \frac{1}{2}\beta(t)\left|1 - \lambda_i^{-1}(t)\right|.
\]
Thus, the stiffness at time $t$ is
\[
\mathrm{stiffness}(t)
=
\frac{\sigma_{\max}(t)}{\sigma_{\min}(t)}
=
\frac{1 - \lambda_n(t)^{-1}}{1 - \lambda_1(t)^{-1}} =
\kappa(\textbf{A}(t)) \cdot
\frac{1 - \lambda_n(t)}{1 - \lambda_1(t)}.
\]

We further know that $\lambda_i(t) = \alpha^2(t) \lambda_i(0) + 1 - \alpha^2(t)$ for some monotonically decreasing  $\alpha^2(t)$ with $\alpha^2(0) = 1$ and $\alpha^2(1) > 0$. Set $a(t)\coloneqq \alpha^2(t)\in(0,1]$ and write $\lambda_i(t) = 1+a(t)(\lambda_i(0)-1)$. Further define 

\[
S(a)
=
\frac{\lambda_1(a)}{\lambda_n(a)}
\cdot
\frac{1 - \lambda_n(a)}{1 - \lambda_1(a)},
\qquad
\lambda_i(a) = 1 + a(\lambda_i(0)-1).
\]
Define $u_i \coloneqq \lambda_i(0) - 1$, so that $u_i < 0$ under the assumption $\lambda_i(0) < 1$, and
\[
\lambda_i(a) = 1 + a u_i, \qquad 1 - \lambda_i(a) = -a u_i.
\]
Hence,
\[
S(a)
=
\frac{1 + a u_1}{1 + a u_n}
\cdot
\frac{-a u_n}{-a u_1}
=
\frac{u_n}{u_1}\cdot \frac{1 + a u_1}{1 + a u_n}.
\]
Taking logarithms,
\[
\log S(a)
=
\log\!\left(\frac{u_n}{u_1}\right)
+
\log(1 + a u_1)
-
\log(1 + a u_n),
\]
and therefore
\begin{align*}
\frac{\mathrm d}{\mathrm d a} \log S(a)
&=
\frac{u_1}{1 + a u_1}
-
\frac{u_n}{1 + a u_n} \\
&=
\frac{u_1(1 + a u_n) - u_n(1 + a u_1)}
{(1 + a u_1)(1 + a u_n)} \\
&=
\frac{u_1 - u_n}{(1 + a u_1)(1 + a u_n)}.
\end{align*}
Since $u_1 - u_n = \lambda_1(0) - \lambda_n(0) \ge 0$ and $\lambda_i(a) = 1 + a u_i > 0$ for all $a \in (0,1]$, the denominator is strictly positive. Hence
\[
\frac{\mathrm d}{\mathrm d a} \log S(a) \ge 0,
\]
so $S(a)$ is non-decreasing in $a$.

As $a(t) = \alpha^2(t)$ is decreasing in $t$, it follows that $S(t)$ is strictly decreasing in $t$. Therefore the maximal stiffness is attained at $t=0$, where $\mathbf A(0) = \mathbf K_n$, giving
\[
S_n
=
\max_{t \in [0,1]} \mathrm{stiffness}(t)
=
\mathrm{stiffness}(0)
=
\kappa(\mathbf K_n)\,
\frac{1 - \lambda_{n,\min}}{1 - \lambda_{n,\max}}.
\]

Finally, since $\sup_n \lambda_{n,\max} < 1$ and $\lambda_{n,\min} \to 0$, the factor
\[
\frac{1 - \lambda_{n,\min}}{1 - \lambda_{n,\max}}
\]
is bounded above and below by positive constants independent of $n$. Hence $S_n \asymp \kappa(\mathbf K_n)$ which establishes the claimed double-sided asymptotic equivalence.
\end{proof}

Theorem~\ref{cor:stiffness} makes clear that the diffusion formulation does not circumvent the fundamental numerical difficulties associated with ill-conditioned GPs. In the prior and linearly conditioned settings, where Theorem~\ref{cor:stiffness} is valid, stiffness is governed by the same spectral quantities that control the condition number of the Gram matrix. In this sense, the diffusion framework provides an alternative sampling mechanism, but it does not eliminate the intrinsic spectral ill-conditioning of the underlying GP. When the conditioning information is no longer linear and Gaussian, an explicit expression for $\textbf{D}(t)$, and hence for stiffness, is typically unavailable. Nevertheless, the notion of stiffness remains well defined. Intuitively, increasingly informative or highly nonlinear conditioning introduces strong anisotropy and local curvature in the conditional score, aligning with the asymptotic regime studied.

\subsection{Transport} 

Stiffness and transport quantify fundamentally different aspects of the generative flow. Stiffness is a local, trajectory-level property: it reflects sensitivity of the ODE to perturbations. Transport, by contrast, is a measure-level property: it quantifies how much probability mass is displaced to transform $p(\bm f_1)$ into $p(\bm f_0)$. 
In effect, stiffness characterises numerical sensitivity of sample paths, but it does not quantify how much mass is rearranged to transform the base distribution into the target, or vice versa. This is instead governed by the geometry of the distributional path and necessitates a measure-level perspective. To make this precise, we're required to explicitly index the time-indexed family of distributions of the process. For the stochastic process \(\{\bm f_t\}_{t\in[0,1]}\), denote by \(p_t\) the density of \(\bm f_t\) at time \(t\). In particular, \(p_0\) and \(p_1\) correspond to the target and base distributions, respectively. We will therefore work with the curve \(\{p_t\}_{t\in[0,1]}\) in the space of probability measures.

We quantify the cost of transporting a distributional path \(\{p_t\}_{t\in[0,1]}\) by its kinetic action. Given a velocity field \(v(\cdot,t)\), e.g. \eqref{eq:GPODE} or \eqref{eq:ODE2}, define
\begin{equation}\label{eq:action}
\mathcal A
:= \frac12\int_0^1 \mathbb E\!\left[\|v(\bm f_t,t)\|^2\right]\,dt
=
\frac12\int_0^1 \int \|v(\bm f,t)\|^2\,p_t(\bm f)\,d\bm f\,dt.
\end{equation}
In general, \cite[Chapter 5]{chewi2025statistical} shows that the pair \(\{p_t,v_t\}_{t\in[0,1]}\) is admissible if it satisfies the continuity equation
\begin{equation}\label{eq:continuity}
\partial_t p_t(\bm f)+\nabla_{\bm f}\cdot\bigl(p_t(\bm f)\,v(\bm f,t)\bigr)=0.
\end{equation}
A given path \(\{p_t\}_{t\in[0,1]}\) may admit multiple velocity fields satisfying \eqref{eq:continuity}, and these may yield different values of the action \eqref{eq:action}. The Benamou--Brenier theorem \cite{benamou2000computational} identifies the minimum possible action.

\begin{theorem}[Benamou--Brenier]\label{thm:BB}
Let \(p_0,p_1\in\mathcal P_{2,\mathrm{ac}}(\mathbb R^d)\). Then
\[
W_2^2(p_0,p_1)
=
\inf_{\{p_t,v_t\}}
\left\{
\int_0^1 \int \|v(\bm f,t)\|^2\,p_t(\bm f)\,d\bm f\,dt
\;\middle|\;
\partial_t p_t+\nabla\cdot(p_tv_t)=0
\right\}
\]
where $W_2^2(p_0,p_1)$ is the squared Wasserstein--2 distance between $p_0$ and $p_1$. In particular, every admissible pair \(\{p_t,v_t\}\) satisfies $W_2^2(p_0,p_1)\le 2\mathcal A$.
\end{theorem}

In particular, in the Gaussian setting considered herein we have $p_t = \mathcal N(0,\textbf A(t))$ and linear velocity field $v(\bm f,t) = -\textbf D(t)\,\bm f$. This leads to Theorem~\ref{cor:gaussian-action} that upper bounds the Wasserstein-2 distance.

\begin{theorem} \label{cor:gaussian-action}
Suppose \(p_t=\mathcal N(0,\textbf A(t))\), where \(\textbf A(t)\in\mathbb R^{n\times n}\) is symmetric positive definite for each \(t\in[0,1]\), and suppose the probability flow velocity $v(\bm f,t) = -\textbf D(t)\,\bm f$, where $\textbf D(t)=\frac12\beta(t)\bigl(I-\textbf A(t)^{-1}\bigr)$. Then
\begin{align*}
W_2^2(p_0,p_1)
\le \frac14\int_0^1 \beta(t)^2
\sum_{i=1}^n
\frac{(\lambda_i(t)-1)^2}{\lambda_i(t)}\,dt,
\end{align*}
where \(\lambda_1(t),\dots,\lambda_n(t)\) are the eigenvalues of \(\textbf A(t)\).
\end{theorem}

\begin{proof}
Since \(\bm f_t \sim \mathcal N(0,\textbf A(t))\) and $v(\bm f_t,t) = -\textbf{D}(t)\bm f_t$,
\[
\mathbb E\!\left[\|v(\bm f_t,t)\|^2\right]
=
\mathbb E\!\left[\bm f_t^\intercal \textbf D(t)^\intercal \textbf D(t)\,\bm f_t\right]
=
\operatorname{Tr}\!\bigl(\textbf D(t)\,\textbf A(t)\,\textbf D(t)^\intercal\bigr),
\]
and therefore
\[
\mathcal A
=
\frac12\int_0^1 \operatorname{Tr}\!\bigl(\textbf D(t)\,\textbf A(t)\,\textbf D(t)^\top\bigr)\,dt.
\]
Substituting \(\textbf{D}(t)=-\tfrac12\beta(t)(\textbf{I}_m - \textbf{A}(t)^{-1})\) gives
\begin{align*}
W_2^2(p_0,p_1)
\le
2 \mathcal A
&=
\int_0^1 \operatorname{Tr}\!\bigl(\textbf D(t)\,\textbf A(t)\,\textbf D(t)^\top\bigr)\,dt \\
&=
\frac14\int_0^1 \beta(t)^2
\Bigl(
\operatorname{Tr}(\textbf A(t)) - 2n + \operatorname{Tr}(\textbf A(t)^{-1})
\Bigr)\,dt \\
&=
\frac14\int_0^1 \beta(t)^2
\sum_{i=1}^n
\Bigl(
\lambda_i(t) - 2 + \lambda_i(t)^{-1}
\Bigr)\,dt \\
&=
\frac14\int_0^1 \beta(t)^2
\sum_{i=1}^n
\frac{(\lambda_i(t)-1)^2}{\lambda_i(t)}\,dt.
\end{align*}
\end{proof}

Finally, this leads to Corollary~\ref{cor:isotropy-minimises} which minimises the upper bound.

\begin{corollary}\label{cor:isotropy-minimises}
Fix a measurable function \(\beta:[0,1]\to\mathbb R\). The upper bound in Theorem~\ref{cor:gaussian-action} is minimised when \(\mathbf A(t)=\textbf{I}\) for all \(t\in[0,1]\). Moreover, the integrand admits the representation
\[
\sum_{i=1}^n \frac{(\lambda_i(t)-1)^2}{\lambda_i(t)} \ge 0,
\]
where \(\lambda_1(t),\dots,\lambda_n(t)\) are the eigenvalues of \(\mathbf A(t)\), with equality if and only if \(\mathbf A(t)=\textbf{I}\) for $\int_0^1 \beta(t)^2 \; \mathrm{d}t > 0$. Consequently, the bound in Theorem~\ref{cor:gaussian-action} vanishes when \(\mathbf A(t)=\textbf{I}\) for all \(t\in[0,1]\).
\end{corollary}

\begin{proof}
For each eigenvalue \(\lambda>0\),
\[
\lambda - 2 + \lambda^{-1}
=
\frac{(\lambda-1)^2}{\lambda} \ge 0,
\]
with equality if and only if \(\lambda=1\). Summing over eigenvalues yields the result.
\end{proof}

\section{General perturbation result}
\label{sec:numerical error}

\textsc{FlowGP} samples by integrating \eqref{eq:NONCONGGPCONDODE} (or its whitened form
\eqref{eq:whitened}) on a finite grid, with the guidance term
$g(t,\bm f):=\nabla_{\bm f}\log p(\mathcal C\mid\bm f,\mathcal D)$ replaced by an
approximation $\widehat g$, such as the Monte Carlo estimator of
Section~\ref{subsec:MC}. There are therefore two sources of error:
\textit{(i)} the guidance approximation $\widehat g\neq g$, and
\textit{(ii)} the ODE solver's discretisation error. We give a deterministic,
pathwise result that controls both errors once a realised approximate guidance field
$\widehat g$ has been fixed. The result is stated on any truncated interval
$[\tau,1]\subset(0,1]$,\footnote{Truncation away from $t=0$ is needed because the bridge
factor $\alpha(t)/(1-\alpha^2(t))$ appearing in the conditional score is unbounded as
$t\downarrow 0$; in practice $\tau$ is taken small (e.g.\ $\tau=10^{-3}$).} under a single
regularity condition on the reversed-time guided drift
\[
\bm a(r,\bm f)
\;=\;
-\bm v(\bm f,1{-}r\mid\mathcal D)
+
\tfrac12\beta(1{-}r)g(1{-}r,\bm f).
\]
The condition is a one-sided Lipschitz bound \citep[see e.g.][]{hairer1993solving}, which
strictly generalises global Lipschitz continuity and admits a constant of either sign.

\begin{theorem}[Pathwise stability of guided sampling]
\label{thm:guidance_stability}
Fix $\tau\in(0,1)$ and assume that, on $[\tau,1]\times\R^m$:
\textup{(i)} $\beta$, $\bm v(\cdot,\cdot\mid\mathcal D)$, $g$, and $\widehat g$ are
continuous;
\textup{(ii)} the reversed-time exact drift $\bm a$ is one-sided Lipschitz, i.e.\ there
exists $\eta_\tau\in\R$ such that
\begin{equation}
\label{eq:OSL_main}
\langle \bm x-\bm y,\,\bm a(r,\bm x)-\bm a(r,\bm y)\rangle
\;\le\;
\eta_\tau\|\bm x-\bm y\|^2
\qquad\forall\bm x,\bm y\in\R^m,\, r\in[0,1{-}\tau];
\end{equation}
\textup{(iii)} the realised guidance error is uniformly bounded,
\[
\varepsilon_\tau
:=
\sup_{(t,\bm f)\in[\tau,1]\times\R^m}
\|g(t,\bm f)-\widehat g(t,\bm f)\|
<\infty .
\]
Let $\bm f_t$ denote the exact guided flow, let $\widehat{\bm f}_t$ denote the exact
solution of the ODE with $g$ replaced by $\widehat g$, and let
$\widehat{\bm f}^{\,h}_{t_n}$ denote the numerical approximation to
$\widehat{\bm f}_{t_n}$ obtained by a deterministic explicit solver of global order $q$
(constant $C_{\mathrm{num},\tau}$, step size $h$). Then for every grid point
$t_n\in[\tau,1]$,
\begin{equation}
\label{eq:guidance_total_error_main}
\|\bm f_{t_n}-\widehat{\bm f}^{\,h}_{t_n}\|
\;\le\;
\tfrac12 B_\tau\varepsilon_\tau\,\Psi_{\eta_\tau}(1-t_n)
\;+\;
C_{\mathrm{num},\tau}\,h^q,
\quad\text{where}\quad
\Psi_\eta(r):=\begin{cases}(e^{\eta r}-1)/\eta, & \eta\neq 0,\\ r, & \eta=0,\end{cases}
\end{equation}
and $B_\tau:=\sup_{t\in[\tau,1]}|\beta(t)|$. 
\end{theorem}

The proof, together with detailed assumptions, is given below. We first discuss the implications of this result for \textsc{FlowGP}.

\paragraph{Three regimes.}
The sign of the constant $\eta_\tau$ in \eqref{eq:OSL_main} determines the qualitative
behaviour of the bound \eqref{eq:guidance_total_error_main}.
\textit{(a) Worst case ($\eta_\tau>0$):} $\Psi_{\eta_\tau}$ grows exponentially in $1-t$,
which is the standard conclusion under any global Lipschitz hypothesis on $\bm v$ and
$g$, with $\eta_\tau=L_v+\tfrac12B_\tau L_g$.
\textit{(b) Borderline ($\eta_\tau=0$):} the perturbation accumulates linearly in $1-t$.
\textit{(c) Contractive ($\eta_\tau<0$):} $\Psi_{\eta_\tau}$ is bounded by $|\eta_\tau|^{-1}$,
so the deterministic perturbation induced by the realised guidance field remains bounded
rather than accumulating, and the total error reduces to
$\tfrac12B_\tau\varepsilon_\tau/|\eta_\tau|+C_{\mathrm{num},\tau}h^q$.

\paragraph{Which likelihoods sit where?}
When $g$ is differentiable, $\nabla_{\bm f}g(t,\bm f) = \nabla^2_{\bm f}\log p(\mathcal C\mid\bm f,\mathcal D),$ so the guidance contribution to $\eta_\tau$ is governed by the curvature of the
smoothed log-likelihood. By Pr\'ekopa--Leindler
\citep{prekopa1973logarithmic}, log-concavity of the base likelihood
$p(\mathcal C\mid\bm f_0)$ is preserved by Gaussian smoothing; see also
\citep{saumard2014log} for a review. Thus log-concave likelihoods give a
non-expansive guidance contribution, while uniformly strongly log-concave
likelihoods give a contractive contribution on the relevant region of state space.
In whitened coordinates, where the Gaussian linear velocity vanishes, this curvature
directly controls the sign of the one-sided Lipschitz constant. In non-whitened
coordinates, the reversed Gaussian flow also contributes to $\eta_\tau$.

Consequently, Gaussian observations and linear inverse problems provide the clearest
examples of a strongly contractive guidance contribution. Probit and logistic
classification, Poisson regression with canonical log link, and convex-set indicator
constraints are log-concave but need not be uniformly strongly log-concave, so they
more naturally correspond to borderline or weakly contractive regimes unless
additional local curvature is available. Genuinely non-log-concave settings, including
mixture observation models, nonlinear forward operators, and the LLM-based
conditioning of Section~\ref{sec:LLM}, are covered only by the worst-case stability
bound. Theorem~\ref{thm:guidance_stability} therefore gives a unified deterministic
perturbation statement, but it should not be read as proving contractivity for every
likelihood used in Section~\ref{sec:experiments}.

\subsection{Proof}

This result proves the finite-horizon deterministic stability bound underlying the guidance error analysis of Section~\ref{sec:experiments}. We control the gap between the trajectories of the exact and approximate guided ODEs on $[\tau,1]$ in terms of a uniform realised guidance error $\varepsilon_\tau$ and a single one-sided Lipschitz constant $\eta_\tau$. The main result, Theorem~\ref{thm:guidance_stability}, specialises to three qualitative regimes according to the sign of $\eta_\tau$: \emph{worst case} ($\eta_\tau>0$, exponential accumulation), \emph{borderline}
($\eta_\tau=0$, linear accumulation), and \emph{contractive} ($\eta_\tau<0$, the guidance perturbation remains bounded). Uniform strong log-concavity of the smoothed likelihood can yield the contractive regime, especially in whitened coordinates, but ordinary log-concavity alone should only be interpreted as a non-expansiveness statement for the guidance contribution.

Fix $\tau\in(0,1)$. On the truncated interval $[\tau,1]$, consider the exact and approximate
guided ODEs
\begin{equation}
\label{eq:app_exact_guided_ode}
\frac{\mathrm d\bm f_t}{\mathrm dt}
=
\bm b(t,\bm f_t),
\qquad
\bm b(t,\bm f)
:=
\bm v(\bm f,t\mid\mathcal D)-\tfrac12\beta(t)g(t,\bm f),
\end{equation}
\begin{equation}
\label{eq:app_approx_guided_ode}
\frac{\mathrm d\widehat{\bm f}_t}{\mathrm dt}
=
\widehat{\bm b}(t,\widehat{\bm f}_t),
\qquad
\widehat{\bm b}(t,\bm f)
:=
\bm v(\bm f,t\mid\mathcal D)-\tfrac12\beta(t)\widehat g(t,\bm f),
\end{equation}
where
\[
g(t,\bm f):=\nabla_{\bm f}\log p(\mathcal C\mid\bm f,\mathcal D)
\]
is the exact guidance field as derived in Section~\ref{sec:nonconjugate}, and $\widehat g$ is its realised approximation. We truncate away from $t=0$ because the bridge factors appearing in the conditional score, such as $\alpha(t)/(1-\alpha^2(t))$, may become singular as $t\downarrow0$.

It will be convenient to phrase regularity on the \emph{reversed-time} drift, since sampling
runs backward from $t=1$, in line with the standard convention for probability-flow ODEs
\citep{anderson1982reverse,songscore}. Define
\begin{equation}
\label{eq:reversed_drifts}
\bm a(r,\bm f)
:=
-\bm b(1-r,\bm f),
\qquad
\widehat{\bm a}(r,\bm f)
:=
-\widehat{\bm b}(1-r,\bm f),
\qquad r\in[0,1-\tau],
\end{equation}
so that
\[
\bm a(r,\bm f)
=
-\bm v(\bm f,1-r\mid\mathcal D)
+
\tfrac12\beta(1-r)g(1-r,\bm f),
\]
and analogously for $\widehat{\bm a}$. Reversed time is set up so that integrating from
$r=0$ corresponds to integrating from $t=1$ down to $t=\tau$.

\begin{assumption}[Uniform regularity of the guided drifts]
\label{ass:app_guidance}
Fix $\tau\in(0,1)$. Assume that:
\begin{enumerate}[leftmargin=*]
    \item $\beta:[\tau,1]\to\mathbb R$ is continuous, with
    \[
    B_\tau:=\sup_{t\in[\tau,1]}|\beta(t)|<\infty;
    \]

    \item $\bm v(\cdot,\cdot\mid\mathcal D)$, $g$, and $\widehat g$ are continuous on
    $[\tau,1]\times\mathbb R^m$, so that the reversed-time drifts $\bm a$ and
    $\widehat{\bm a}$ defined in \eqref{eq:reversed_drifts} are continuous on
    $[0,1-\tau]\times\mathbb R^m$;

    \item the reversed-time exact drift $\bm a$ satisfies a one-sided Lipschitz condition
    in $\bm f$, uniformly in $r$: therefore there exists $\eta_\tau\in\mathbb R$ such that
    \begin{equation}
    \label{eq:app_OSL}
    \langle \bm x-\bm y,\;\bm a(r,\bm x)-\bm a(r,\bm y)\rangle
    \le
    \eta_\tau\|\bm x-\bm y\|^2
    \qquad
    \forall \bm x,\bm y\in\mathbb R^m,\ r\in[0,1-\tau];
    \end{equation}

    \item the realised guidance approximation error is uniformly bounded on the relevant state-space region visited by the exact and approximate trajectories:
    \[
    \varepsilon_\tau
    :=
    \sup_{(t,\bm f)\in\mathcal R_\tau}
    \|g(t,\bm f)-\widehat g(t,\bm f)\|
    <\infty,
    \]
    where $\mathcal R_\tau\subseteq[\tau,1]\times\mathbb R^m$ denotes any region containing both trajectories on the truncated interval.
\end{enumerate}
\end{assumption}

The constant $\eta_\tau$ is allowed to be negative; this contractive case will be exploited in Corollary~\ref{cor:app_strongly_monotone} below. The uniform error condition should be interpreted as a deterministic perturbation assumption on the realised guidance field over the relevant state-space region visited by the sampler. When $\widehat g$ is computed by Monte Carlo, all bounds below are conditional on the realised Monte Carlo randomness unless additional stochastic uniform-error bounds are imposed. 

For convenience, we define
\begin{equation}
\label{eq:app_guidance_constants}
C_\tau
:=
\tfrac12 B_\tau \varepsilon_\tau,
\qquad
\Psi_\eta(r)
:=
\begin{cases}
\dfrac{e^{\eta r}-1}{\eta}, & \eta\neq 0,\\[1ex]
r, & \eta=0,
\end{cases}
\quad r\ge 0.
\end{equation}

\begin{lemma}[Properties of $\Psi_\eta$]
\label{lem:app_psi_properties}
The map $(\eta,r)\mapsto\Psi_\eta(r)$ defined by \eqref{eq:app_guidance_constants}
satisfies:
\begin{enumerate}[label=(\roman*),leftmargin=*]
    \item $\Psi_\eta(0)=0$ and $r\mapsto\Psi_\eta(r)$ is continuously differentiable with
    derivative $e^{\eta r}>0$, and is therefore strictly increasing;
    \item $(\eta,r)\mapsto\Psi_\eta(r)$ is jointly continuous on $\mathbb R\times[0,\infty)$;
    in particular $\Psi_\eta(r)\to r$ as $\eta\to 0$;
    \item if $\eta<0$, then
    \[
    \Psi_\eta(r)=\frac{1-e^{-|\eta|r}}{|\eta|}
    \le
    \frac{1}{|\eta|}
    \qquad\forall r\ge 0,
    \]
    so $\Psi_\eta$ is bounded uniformly in $r$.
\end{enumerate}
\end{lemma}
\begin{proof}
Direct calculation. Joint continuity at $\eta=0$ follows from the Taylor expansion
$(e^{\eta r}-1)/\eta=r+\tfrac12\eta r^2+O(\eta^2)$; the bound for $\eta<0$ uses
$1-e^{-|\eta|r}\le 1$.
\end{proof}

\begin{remark}[How to verify Assumption~\ref{ass:app_guidance}\,(3)]
\label{rem:app_OSL_specialisations}
The one-sided Lipschitz condition \eqref{eq:app_OSL} can be checked through a hierarchy of
increasingly favourable regimes.
\begin{enumerate}[label=(\roman*),leftmargin=*]
    \item \emph{Global Lipschitz drift.}
    If $\bm b$ is globally Lipschitz in $\bm f$, uniformly in $t\in[\tau,1]$, with constant
    $L_\tau$, then $\bm a(r,\cdot)=-\bm b(1-r,\cdot)$ is also globally Lipschitz in $\bm f$
    with the same constant. Hence, by Cauchy--Schwarz,
    \[
    \langle \bm x-\bm y,\;\bm a(r,\bm x)-\bm a(r,\bm y)\rangle
    \le
    \|\bm x-\bm y\|\,\|\bm a(r,\bm x)-\bm a(r,\bm y)\|
    \le
    L_\tau\|\bm x-\bm y\|^2.
    \]
    Thus \eqref{eq:app_OSL} holds with $\eta_\tau=L_\tau$. In particular, if $\bm v$ and
    $g$ are globally Lipschitz in $\bm f$, uniformly in $t$, with constants $L_v$ and $L_g$, respectively, then one can take
    \[
    \eta_\tau
    =
    L_v+\tfrac12 B_\tau L_g.
    \]

    \item \emph{One-sided Lipschitz guidance.}
    Suppose $\beta(t)\ge0$ on $[\tau,1]$. Suppose also that $g$ is one-sided Lipschitz with
    constant $\ell_\tau\in\mathbb R$, uniformly in $t$, namely
    \[
    \langle \bm x-\bm y,\;g(t,\bm x)-g(t,\bm y)\rangle
    \le
    \ell_\tau\|\bm x-\bm y\|^2,
    \]
    and that $-\bm v(\cdot,1-r\mid\mathcal D)$ is one-sided Lipschitz with constant
    $\eta_{v,\tau}$, uniformly in $r$. Decomposing gives
    \[
    \bm a(r,\bm f)
    =
    \underbrace{-\bm v(\bm f,1-r\mid\mathcal D)}_{\text{Gaussian flow}}
    +
    \underbrace{\tfrac12\beta(1-r)g(1-r,\bm f)}_{\text{guidance term}},
    \]
    and adding the two one-sided Lipschitz estimates, one obtains
    \[
    \langle \bm x-\bm y,\bm a(r,\bm x)-\bm a(r,\bm y)\rangle
    \le
    \eta_{v,\tau}\|\bm x-\bm y\|^2
    +
    \tfrac12\beta(1-r)\ell_\tau\|\bm x-\bm y\|^2.
    \]
    Taking the supremum over $r\in[0,1-\tau]$ shows that \eqref{eq:app_OSL} holds with
    \[
    \eta_\tau
    =
    \eta_{v,\tau}
    +
    \tfrac12
    \sup_{t\in[\tau,1]}\{\beta(t)\ell_\tau\}.
    \]
    In particular, if $\bm v$ is globally Lipschitz with constant $L_v$, one may take
    $\eta_{v,\tau}=L_v$. Writing
    \[
    \overline B_\tau:=\sup_{t\in[\tau,1]}\beta(t),
    \qquad
    \underline B_\tau:=\inf_{t\in[\tau,1]}\beta(t),
    \]
    the supremum is attained at the larger or smaller value of $\beta$ depending on the sign
    of $\ell_\tau$:
    \[
    \eta_\tau
    =
    L_v+\tfrac12\overline B_\tau\ell_\tau
    \qquad\text{if } \ell_\tau\ge0,
    \qquad
    \eta_\tau
    =
    L_v+\tfrac12\underline B_\tau\ell_\tau
    \qquad\text{if } \ell_\tau<0.
    \]
    The latter case is the one in which contractivity of the guidance can improve the
    perturbation bound.

    \item \emph{Strongly monotone guidance.}
    Suppose $\bm v\equiv0$, $\beta(t)\ge\underline B_\tau>0$ on $[\tau,1]$, and $g$ satisfies
    the strong monotonicity condition
    \[
    \langle \bm x-\bm y,\;g(t,\bm x)-g(t,\bm y)\rangle
    \le
    -\mu_\tau\|\bm x-\bm y\|^2
    \]
    for some $\mu_\tau>0$, uniformly in $t\in[\tau,1]$. Then
    \eqref{eq:app_OSL} holds with
    \[
    \eta_\tau
    =
    -\tfrac12\underline B_\tau\mu_\tau<0,
    \]
    and the perturbation bound saturates; see
    Corollary~\ref{cor:app_strongly_monotone}.
\end{enumerate}
\end{remark}

\begin{proposition}[Uniform pathwise stability under approximate guidance]
\label{prop:app_guidance_pathwise}
Let Assumption~\ref{ass:app_guidance} hold. Suppose that the exact and approximate
reversed-time ODEs admit solutions on $[0,1-\tau]$ from a common initial condition, or
equivalently that the original-time ODEs
\eqref{eq:app_exact_guided_ode}--\eqref{eq:app_approx_guided_ode} admit solutions on
$[\tau,1]$ with common terminal condition $\bm f_1=\widehat{\bm f}_1$. Then, for every
$t\in[\tau,1]$,
\begin{equation}
\label{eq:app_guidance_pathwise_bound}
\|\bm f_t-\widehat{\bm f}_t\|
\le
C_\tau\,\Psi_{\eta_\tau}(1-t),
\end{equation}
which implies,
\begin{equation}
\label{eq:app_guidance_pathwise_sup}
\sup_{t\in[\tau,1]}\|\bm f_t-\widehat{\bm f}_t\|
\le
C_\tau\,\Psi_{\eta_\tau}(1-\tau).
\end{equation}
Equivalently,
\[
\|\bm f_t-\widehat{\bm f}_t\|
\le
\begin{cases}
\dfrac{B_\tau\varepsilon_\tau}{2\eta_\tau}
\bigl(e^{\eta_\tau(1-t)}-1\bigr),
& \eta_\tau\neq 0,\\[1.5ex]
\tfrac12 B_\tau\varepsilon_\tau\,(1-t),
& \eta_\tau=0,
\end{cases}
\]

\end{proposition}

\begin{proof}
Switch to reversed time $r=1-t$, and define
\[
\bm y_r:=\bm f_{1-r},
\qquad
\widehat{\bm y}_r:=\widehat{\bm f}_{1-r},
\qquad
\bm e_r:=\bm y_r-\widehat{\bm y}_r,
\qquad r\in[0,1-\tau].
\]
A change of variables in \eqref{eq:app_exact_guided_ode}--\eqref{eq:app_approx_guided_ode}
gives
\[
\frac{\mathrm d\bm y_r}{\mathrm dr}=\bm a(r,\bm y_r),
\qquad
\frac{\mathrm d\widehat{\bm y}_r}{\mathrm dr}
=
\widehat{\bm a}(r,\widehat{\bm y}_r),
\qquad
\bm e_0=0.
\]
Hence
\begin{equation}
\label{eq:app_e_prime}
\bm e_r'
=
\bm a(r,\bm y_r)-\widehat{\bm a}(r,\widehat{\bm y}_r).
\end{equation}

The map $r\mapsto\|\bm e_r\|^2$ is differentiable but $r\mapsto\|\bm e_r\|$ is not, near
points where $\bm e_r=\bm 0$. To circumvent this, we work with the strictly positive smooth
surrogate
\[
\phi_\varepsilon(r):=\sqrt{\|\bm e_r\|^2+\varepsilon},
\qquad \varepsilon>0,
\]
which is continuously differentiable for every $\varepsilon>0$ and satisfies
$\phi_\varepsilon(0)=\sqrt{\varepsilon}$ (since $\bm e_0=\bm 0$), and therefore it follows that
\begin{equation}
\label{eq:app_phi_prime}
\phi_\varepsilon'(r)
=
\frac{1}{2\phi_\varepsilon(r)}\frac{\mathrm d}{\mathrm dr}\|\bm e_r\|^2
=
\frac{\langle \bm e_r,\bm e_r'\rangle}{\phi_\varepsilon(r)}.
\end{equation}
We will derive a linear differential inequality for $\phi_\varepsilon$, integrate it using the
Gr\"onwall inequality, and then let $\varepsilon\downarrow 0$.

We start by adding and subtracting $\bm a(r,\widehat{\bm y}_r)$ in \eqref{eq:app_e_prime} to split $\bm e_r'$ into a
contractive part and a drift error:
\[
\bm e_r'
=
\underbrace{\bm a(r,\bm y_r)-\bm a(r,\widehat{\bm y}_r)}_{\text{exact-drift difference}}
+
\underbrace{\bm a(r,\widehat{\bm y}_r)-\widehat{\bm a}(r,\widehat{\bm y}_r)}_{\text{drift error}}.
\]
Applying the one-sided Lipschitz assumption~\eqref{eq:app_OSL} to the first term with
$\bm x=\bm y_r$ and $\bm y=\widehat{\bm y}_r$,
\begin{equation}
\label{eq:app_contraction_bound}
\langle \bm e_r,\bm a(r,\bm y_r)-\bm a(r,\widehat{\bm y}_r)\rangle
\le
\eta_\tau\|\bm e_r\|^2.
\end{equation}

The Gaussian flow $\bm v$ enters $\bm b$ and $\widehat{\bm b}$ identically and so cancels in
$\bm a-\widehat{\bm a}$, meaning that for the second term we have,
\[
\bm a(r,\widehat{\bm y}_r)-\widehat{\bm a}(r,\widehat{\bm y}_r)
=
\tfrac12\beta(1-r)
\bigl(
g(1-r,\widehat{\bm y}_r)
-
\widehat g(1-r,\widehat{\bm y}_r)
\bigr).
\]
By Cauchy--Schwarz and Assumption~\ref{ass:app_guidance}\,(1) and (4), we have
\begin{equation}
\label{eq:app_defect_bound}
\bigl|
\langle \bm e_r,
\bm a(r,\widehat{\bm y}_r)-\widehat{\bm a}(r,\widehat{\bm y}_r)
\rangle
\bigr|
\le
\|\bm e_r\|\cdot\tfrac12 B_\tau\varepsilon_\tau
=
C_\tau\|\bm e_r\|.
\end{equation}

Combining \eqref{eq:app_contraction_bound} and \eqref{eq:app_defect_bound} gives,
\begin{equation}
\label{eq:app_inner_product_bound}
\langle \bm e_r,\bm e_r'\rangle
\le
\eta_\tau\|\bm e_r\|^2
+
C_\tau\|\bm e_r\|.
\end{equation}
We divide by $\phi_\varepsilon(r)>0$ and bound each term.
For the drift error term,
\[
C_\tau\frac{\|\bm e_r\|}{\phi_\varepsilon(r)}
=
C_\tau\frac{\|\bm e_r\|}{\sqrt{\|\bm e_r\|^2+\varepsilon}}
\le
C_\tau,
\]
since $\|\bm e_r\|\le\sqrt{\|\bm e_r\|^2+\varepsilon}$.
For the contraction term we use the identity
$\|\bm e_r\|^2=\phi_\varepsilon(r)^2-\varepsilon$ to rewrite
\begin{equation}
\label{eq:app_quadratic_identity}
\frac{\eta_\tau\|\bm e_r\|^2}{\phi_\varepsilon(r)}
=
\eta_\tau\phi_\varepsilon(r)
-
\frac{\eta_\tau\varepsilon}{\phi_\varepsilon(r)}.
\end{equation}
The second term is bounded uniformly in $r$ by
\[
\left|\frac{\eta_\tau\varepsilon}{\phi_\varepsilon(r)}\right|
\le
\frac{|\eta_\tau|\varepsilon}{\sqrt{\varepsilon}}
=
|\eta_\tau|\sqrt{\varepsilon},
\]
since $\phi_\varepsilon(r)\ge\sqrt{\varepsilon}$. The bound
$|{-}\eta_\tau\varepsilon/\phi_\varepsilon|\le|\eta_\tau|\sqrt{\varepsilon}$ is what allows
the argument to handle all signs of $\eta_\tau$ uniformly: when $\eta_\tau\ge0$ the simpler
estimate $\eta_\tau\|\bm e_r\|^2/\phi_\varepsilon\le\eta_\tau\phi_\varepsilon$ already
suffices, but when $\eta_\tau<0$ this naive bound goes the wrong way and the
$|\eta_\tau|\sqrt{\varepsilon}$ correction is needed. Combining gives,
\begin{equation}
\label{eq:app_diff_inequality}
\phi_\varepsilon'(r)
\le
\eta_\tau\phi_\varepsilon(r)
+
C_\tau
+
|\eta_\tau|\sqrt{\varepsilon}.
\end{equation}

Inequality \eqref{eq:app_diff_inequality} is a linear differential inequality with constant
coefficient $\eta_\tau$ and forcing $K_\varepsilon:=C_\tau+|\eta_\tau|\sqrt{\varepsilon}$.
We integrate by the standard integrating-factor argument (the linear instance of Gr\"onwall;
see e.g.\ \citealp[Lemma~5.16]{chewi2025statistical}). Multiplying both sides of
\eqref{eq:app_diff_inequality} by $e^{-\eta_\tau r}>0$,
\[
\frac{\mathrm d}{\mathrm dr}
\bigl(e^{-\eta_\tau r}\phi_\varepsilon(r)\bigr)
=
e^{-\eta_\tau r}\bigl(\phi_\varepsilon'(r)-\eta_\tau\phi_\varepsilon(r)\bigr)
\le
K_\varepsilon e^{-\eta_\tau r}.
\]
Integrating from $0$ to $r$ and using $\phi_\varepsilon(0)=\sqrt{\varepsilon}$,
\[
e^{-\eta_\tau r}\phi_\varepsilon(r)-\sqrt{\varepsilon}
\le
K_\varepsilon\int_0^r e^{-\eta_\tau s}\,\mathrm ds.
\]
A direct computation gives
\[
\int_0^r e^{-\eta_\tau s}\,\mathrm ds
=
\begin{cases}
\dfrac{1-e^{-\eta_\tau r}}{\eta_\tau}, & \eta_\tau\neq 0,\\[1ex]
r, & \eta_\tau=0,
\end{cases}
\quad
=
e^{-\eta_\tau r}\,\Psi_{\eta_\tau}(r),
\]
where the second equality uses the definition of $\Psi_{\eta_\tau}$ in
\eqref{eq:app_guidance_constants}. Multiplying through by $e^{\eta_\tau r}$,
\[
\phi_\varepsilon(r)
\le
\sqrt{\varepsilon}\,e^{\eta_\tau r}
+
K_\varepsilon\,\Psi_{\eta_\tau}(r),
\qquad r\in[0,1-\tau].
\]

Since $\|\bm e_r\|\le\phi_\varepsilon(r)$ for every $\varepsilon>0$, and since
$K_\varepsilon=C_\tau+|\eta_\tau|\sqrt{\varepsilon}\to C_\tau$ as $\varepsilon\downarrow 0$,
the right-hand side is continuous in $\varepsilon$ at $0$ for each fixed $r$, with
\[
\lim_{\varepsilon\downarrow 0}
\bigl(\sqrt{\varepsilon}\,e^{\eta_\tau r}+K_\varepsilon\Psi_{\eta_\tau}(r)\bigr)
=
C_\tau\Psi_{\eta_\tau}(r).
\]
Therefore $\|\bm e_r\|\le C_\tau\Psi_{\eta_\tau}(r)$.

Substituting back $r=1-t$ proves \eqref{eq:app_guidance_pathwise_bound}. The supremum bound
\eqref{eq:app_guidance_pathwise_sup} is immediate from
Lemma~\ref{lem:app_psi_properties}\,(i).
\end{proof}

\begin{corollary}[Stability under globally Lipschitz drifts]
\label{cor:app_lipschitz}
Suppose $\bm v$ and $g$ are globally Lipschitz in $\bm f$, uniformly in $t\in[\tau,1]$, with
constants $L_v$ and $L_g$ respectively. Then Proposition~\ref{prop:app_guidance_pathwise}
holds with
\[
\eta_\tau
=
L_\tau
:=
L_v+\tfrac12B_\tau L_g
\ge0.
\]
Consequently,
\[
\|\bm f_t-\widehat{\bm f}_t\|
\le
\frac{B_\tau\varepsilon_\tau}{2L_\tau}
\bigl(e^{L_\tau(1-t)}-1\bigr)
\quad (L_\tau>0),
\]
whereas, if $L_\tau=0$, then
\[
\|\bm f_t-\widehat{\bm f}_t\|
\le
\tfrac12B_\tau\varepsilon_\tau(1-t).
\]
\end{corollary}

\begin{proof}
Follows immediately from Remark~\ref{rem:app_OSL_specialisations}\,(i) and
Proposition~\ref{prop:app_guidance_pathwise}.
\end{proof}

\begin{corollary}[Stability under strongly monotone guidance]
\label{cor:app_strongly_monotone}
Assume the whitened setting $\bm v\equiv0$, with
$\beta(t)\ge\underline B_\tau>0$ on $[\tau,1]$, and suppose $g$ satisfies
\[
\langle \bm x-\bm y,\;g(t,\bm x)-g(t,\bm y)\rangle
\le
-\mu_\tau\|\bm x-\bm y\|^2,
\qquad
\forall t\in[\tau,1],\ \bm x,\bm y\in\mathbb R^m,
\]
for some $\mu_\tau>0$. Then Proposition~\ref{prop:app_guidance_pathwise} applies with
\[
\eta_\tau
=
-\tfrac12\underline B_\tau\mu_\tau<0,
\]
and
\[
\|\bm f_t-\widehat{\bm f}_t\|
\le
\frac{B_\tau\varepsilon_\tau}{\underline B_\tau\mu_\tau}
\left(
1-e^{-\frac12\underline B_\tau\mu_\tau(1-t)}
\right),
\qquad t\in[\tau,1].
\]
In particular,
\[
\sup_{t\in[\tau,1]}\|\bm f_t-\widehat{\bm f}_t\|
\le
\frac{B_\tau\varepsilon_\tau}{\underline B_\tau\mu_\tau},
\]
so the perturbation saturates at a level proportional to $\varepsilon_\tau/\mu_\tau$, for
fixed truncated-interval constants.
\end{corollary}

\begin{proof}
With $\bm v\equiv0$,
\[
\bm a(r,\bm f)
=
\tfrac12\beta(1-r)g(1-r,\bm f),
\]
so for any $\bm x,\bm y\in\mathbb R^m$ and $r\in[0,1-\tau]$,
\[
\langle \bm x-\bm y,\bm a(r,\bm x)-\bm a(r,\bm y)\rangle
=
\tfrac12\beta(1-r)
\langle \bm x-\bm y,
g(1-r,\bm x)-g(1-r,\bm y)
\rangle.
\]
By the strong monotonicity of $g$ and $\beta(1-r)\ge\underline B_\tau>0$,
\[
\langle \bm x-\bm y,\bm a(r,\bm x)-\bm a(r,\bm y)\rangle
\le
-\tfrac12\underline B_\tau\mu_\tau
\|\bm x-\bm y\|^2.
\]
Hence \eqref{eq:app_OSL} holds with
$\eta_\tau=-\tfrac12\underline B_\tau\mu_\tau<0$. Substituting this value into
Proposition~\ref{prop:app_guidance_pathwise} and using
$\Psi_\eta(r)=(1-e^{-|\eta|r})/|\eta|$ for $\eta<0$
(Lemma~\ref{lem:app_psi_properties}\,(iii)) gives the required bound.
\end{proof}

\begin{remark}[Comparison of regimes]
\label{rem:app_three_regimes}
The unified bound \eqref{eq:app_guidance_pathwise_bound} interpolates between three
qualitatively distinct behaviours, governed by the sign of $\eta_\tau$.
\begin{itemize}[leftmargin=*]
    \item \emph{Expansive case $(\eta_\tau>0)$.}
    The bound grows exponentially in the backward integration horizon $1-t$. This is the
    worst-case behaviour obtained from a purely global Lipschitz argument.

    \item \emph{Borderline case $(\eta_\tau=0)$.}
    The perturbation accumulates linearly:
    \[
    \|\bm f_t-\widehat{\bm f}_t\|
    \le
    C_\tau(1-t).
    \]

    \item \emph{Contractive case $(\eta_\tau<0)$.}
    The pathwise error satisfies
    \[
    \|\bm f_t-\widehat{\bm f}_t\|
    \le
    \frac{C_\tau}{|\eta_\tau|}
    \left(1-e^{-|\eta_\tau|(1-t)}\right)
    \le
    \frac{C_\tau}{|\eta_\tau|}.
    \]
    Thus the approximation error saturates rather than accumulating indefinitely over the
    backward integration horizon. The constants may still depend on the truncation level
    $\tau$.
\end{itemize}
The one-sided Lipschitz analysis can therefore yield sharper bounds than a global
Lipschitz analysis whenever the effective reversed-time drift is monotone or contractive.
In the GP setting, this behaviour is expected when the effective reversed-time drift is contractive over the region visited by the sampler. In whitened coordinates this can occur when the smoothed conditioning likelihood is uniformly strongly log-concave over that
region.
\end{remark}

\subsection{Adding deterministic ODE discretisation error}

In practice, the approximate ODE \eqref{eq:app_approx_guided_ode} is integrated numerically
on a backward grid
\[
1=t_0>t_1>\cdots>t_N=\tau,
\qquad
h:=\max_{0\le n\le N-1}(t_n-t_{n+1}).
\]
We assume that the chosen solver is a deterministic explicit Runge--Kutta-type method (e.g.\
explicit Euler, Heun's method, classical RK4, or one of the diffusion-tailored variants
such as DPM-Solver~\citep{lu2022dpm}, all of which are deterministic schemes for
deterministic ODEs; see \citealp[Chapter~9]{lai2025principles} for a survey). The choice of solver enters through its global error along the relevant trajectory. Here $\widehat{\bm f}_t$ denotes the exact solution of the approximately guided ODE \eqref{eq:app_approx_guided_ode}, whereas $\widehat{\bm f}^{\,h}_{t_n}$ denotes its numerical approximation on the grid. We assume the solver has global order $q$:
\begin{equation}
\label{eq:app_num_global_error}
\max_{0\le n\le N}
\|\widehat{\bm f}_{t_n}-\widehat{\bm f}^{\,h}_{t_n}\|
\le
C_{\mathrm{num},\tau}\,h^q.
\end{equation}
Such bounds are classical in numerical analysis; for example, $q=1$ for explicit Euler,
$q=2$ for Heun's method, $q=4$ for RK4
\citep[Section~2.6]{sarkka2019applied}, and $q\ge 2$ for the DPM-Solver family
\citep{lu2022dpm}.\footnote{The constant $C_{\mathrm{num},\tau}$ depends on the truncated
interval $[\tau,1]$, on bounds for $\widehat{\bm a}$ and its higher derivatives along the
trajectory, and on the solver. We use $q$ for the solver order to avoid clashing with the
Wasserstein exponent $p$.}

\begin{proposition}[Total error: guidance approximation plus discretisation]
\label{prop:app_total_error}
Under Assumption~\ref{ass:app_guidance}, the existence assumptions of
Proposition~\ref{prop:app_guidance_pathwise}, and the numerical global-error bound
\eqref{eq:app_num_global_error}, for every grid point $t_n\in[\tau,1]$,
\[
\|\bm f_{t_n}-\widehat{\bm f}^{\,h}_{t_n}\|
\le
C_\tau\,\Psi_{\eta_\tau}(1-t_n)
+
C_{\mathrm{num},\tau}\,h^q.
\]
Consequently,
\[
\max_{0\le n\le N}
\|\bm f_{t_n}-\widehat{\bm f}^{\,h}_{t_n}\|
\le
C_\tau\,\Psi_{\eta_\tau}(1-\tau)
+
C_{\mathrm{num},\tau}\,h^q.
\]
\end{proposition}

\begin{proof}
By the triangle inequality applied to the three flows
$\bm f_{t_n}\to\widehat{\bm f}_{t_n}\to\widehat{\bm f}^{\,h}_{t_n}$,
\[
\|\bm f_{t_n}-\widehat{\bm f}^{\,h}_{t_n}\|
\le
\underbrace{\|\bm f_{t_n}-\widehat{\bm f}_{t_n}\|}_{\le C_\tau\Psi_{\eta_\tau}(1-t_n)
\text{ by Prop.~\ref{prop:app_guidance_pathwise}}}
+
\underbrace{\|\widehat{\bm f}_{t_n}-\widehat{\bm f}^{\,h}_{t_n}\|}_{\le C_{\mathrm{num},\tau}h^q
\text{ by \eqref{eq:app_num_global_error}}}.
\]
The supremum bound follows from monotonicity of $\Psi_{\eta_\tau}$
(Lemma~\ref{lem:app_psi_properties}\,(i)).
\end{proof}

\begin{remark}[Self-normalised Monte Carlo guidance]
\label{rem:app_self_normalised_MC}
The Monte Carlo estimator \eqref{eq:guidance-mc-estimator} is a self-normalised
importance sampler and is therefore biased at finite $S$. At a fixed deterministic
state $(t,\bm f)$, standard self-normalised importance sampling results give, under
appropriate moment conditions on the weights and weighted scores,
\[
\E[\widehat g_S(t,\bm f)] - g(t,\bm f) = \mathcal O(S^{-1}),
\qquad
\mathrm{Var}(\widehat g_S(t,\bm f)) = \mathcal O(S^{-1}),
\]
with constants depending on the discrepancy between
$p(\bm f_0\mid\bm f_t,\mathcal D)$ and
$p(\bm f_0\mid\bm f_t,\mathcal C,\mathcal D)$, for example through
importance-weight moment or chi-squared-divergence quantities.

These pointwise Monte Carlo statements do not by themselves imply a trajectory-level
$\mathcal O(S^{-1/2})$ error bound for the sampler. The numerical trajectory is evaluated
at states that depend on the realised Monte Carlo guidance field. This dependence is
particularly explicit when common random numbers are drawn once before integration and
then reused across ODE steps. In that case, the trajectory is coupled to the empirical
error field
\[
(t,\bm f)\mapsto \widehat g_S(t,\bm f)-g(t,\bm f).
\]
A stochastic version of Theorem~\ref{thm:guidance_stability} would therefore require
additional control of a uniform empirical-process quantity, such as
\[
\E\!\left[
\sup_{(t,\bm f)\in\mathcal R_\tau}
\|\widehat g_S(t,\bm f)-g(t,\bm f)\|
\right]
\quad\text{or a corresponding high-probability bound}.
\]
Such a bound need not follow from the pointwise variance rate without further assumptions
on the likelihood, the importance weights, and the complexity of the class of guidance
functions encountered along the trajectory. Accordingly, the perturbation bound above is
used pathwise: conditional on a realised approximate guidance field satisfying
Assumption~\ref{ass:app_guidance}\,(4), the resulting trajectory error is controlled by
Theorem~\ref{thm:guidance_stability}. Establishing sharp stochastic trajectory-level Monte
Carlo rates is left for future work.
\end{remark}

\section{Practical algorithm details}
\label{app:method}

\subsection{Schedules}
\label{appendix:schedule}

We define a VP diffusion process over whitened samples.
The forward process marginals are
\begin{equation}
    \hat{\mathbf{f}}_t = \alpha(t)\,\hat{\mathbf{f}}_0 + \sqrt{1-\alpha^2(t)}\,\boldsymbol{z},
    \qquad \boldsymbol{z} \sim \mathcal{N}(\mathbf{0}, \mathbf{I}_m),
    \label{eq:forward}
\end{equation}
where the signal-to-noise decay is governed by a linear $\beta$-schedule,
\begin{equation}
    \beta(t) = \beta_0 + (\beta_1 - \beta_0)\,t,
    \qquad
    \alpha(t) = \exp\!\left(-\frac{1}{2}\int_0^t \beta(s)\,\text{d}s\right)
              = \exp\!\left(-\tfrac{1}{2}\beta_0 t - \tfrac{1}{4}(\beta_1-\beta_0)t^2\right),
\end{equation}
with $\beta_0 = 10^{-5}$ and $\beta_1 = 10.0$. These endpoints are chosen to ensure $\alpha(1)\approx0$ so that the initial condition is effectively pure noise, whilst ensuring that $\alpha(0)=1$ so that the terminal distribution recovers the GP posterior exactly.

\subsection{Time Discretisation}
\label{appendix:snr}
We integrate from $t = 1$ to $t \approx 0$ using a schedule that is uniform
in \emph{log-SNR} space, where $\mathrm{SNR}(t) = \alpha(t)/\sqrt{(1-\alpha^2(t) +1 \times 10^{-8})}$.
Binary search is used to invert the SNR function and place the $T+1$ grid points
$t_0 = 1 > t_1 > \cdots > t_T \approx 0$ at equal intervals in $\log \mathrm{SNR}$.
This concentrates steps where the SNR changes most rapidly, improving
integration quality without increasing $T$.

\subsection{Practical Guidance Calculation}
\label{app:practical_guidance}
 To reduce variance between ODE steps, we use the reparameterisation trick, sampling $S$ isotropic Gaussian samples $\boldsymbol{ \epsilon}^{(i)}\sim\mathcal{N}(\bm{0}, \mathbf{I}_m)$ before beginning integration. At each ODE step, we
\begin{enumerate}
    \item Build $S$ samples $\bm f_0^{(i)}\sim p(\bm f_0 \mid \bm f_t, \mathcal{D})$ via $\bm f_0^{(i)} = \bm \mu_{0|t} + \bm \Sigma_{0|t}^{1/2}\bm \epsilon^{(i)}$, with $\bm \mu_{0|t}$ from \eqref{eq:bridge_mean_notes} and $\boldsymbol{\Sigma}_{0|t}$ from \eqref{eq:bridge_cov_notes} with $\mathbf{m}_{*|\bm{y}}$, $\mathbf{K}_{**|\bm{y}}$ and $\mathbf{A}_{|\bm{y}}(t)$ in place of $\mathbf{m}_*, \mathbf{K}_{**}$ and $\mathbf{A}(t)$.
    \item Evaluate each log-likelihood $\log p(\mathcal{C}\mid \bm f_0^{(i)})$ and its gradient $\nabla_{\bm f_0}\log p(\mathcal{C}\mid \bm f_0^{(i)})$ at each sample,
    \item Compute normalised weights via the numerically stable log-sum-exp operation $\log \bar w^{(i)} = \log p(\mathcal{C}\mid\bm f_0^{(i)}) - \operatorname{logsumexp}_r \log p(\mathcal{C}\mid\bm f_0^{(r)})$. Here $$\operatorname{logsumexp}_ra^{(r)} := \max_r a^{(r)} + \log\left[\sum_r\exp\left(a^{(r)} - \max_r a^{(r)}\right)\right]$$ avoids overflow and underflow by subtracting the potentially very small maximal value.
    \item Evaluate the weighted sum \eqref{eq:finalform}, applying the Jacobian to each gradient term,
    \item We clip the norm of the vector field 
    (after scaling by $-\frac{1}{2}\beta(t)$ as prescribed by the probability flow ODE \eqref{eq:NONCONGGPCONDODE}) to limit excessively large steps to ensure stable integration. We use the smooth saturation: $\bm{v} \mapsto \bm{v} \cdot \tau \tanh(\|\bm{v}\|/\tau) / (\|\bm{v}\|+1e^{-8})$, where $\tau=1\times 10^2$ is a maximum norm threshold. This transformation bounds excessively large gradients whilst preserving Lipschitz continuity.
\end{enumerate}

\subsection{Full integration loop}

The full sampling loop integrates the probability-flow ODE with the Euler method.
Since the unconditional vector field is zero in whitened coordinates, each step
reduces to applying only the conditional correction, see Algorithm \ref{alg:sampling}.

\newcommand{\algphasegap}{\STATEx\vspace{0.3ex}}
\begin{algorithm}[t]
\caption{\textsc{FlowGP}: sampling from a GP predictive distribution under
arbitrary conditioning via the whitened probability-flow ODE}
\label{alg:sampling}
\begin{algorithmic}[1]

    \algphasegap
    
\REQUIRE
    Linear-Gaussian predictive mean $\mathbf{m}_{*|\bm y}$ and covariance
    $\mathbf{K}_{**|\bm y}$ from the data $\mathcal{D}$;
    non-Gaussian condition $\mathcal{C}$ with point-wise evaluable likelihood
    $p(\mathcal{C}\mid\bm f_0)$;
    number of ODE steps $T$;
    number of Monte Carlo samples $S$;
    smooth-clipping threshold $v_{\max}$;
    schedule $\alpha(t),\beta(t)$ (Appendix~\ref{appendix:schedule}).
    \algphasegap
\ENSURE Sample $\bm f_0 \sim p(\bm f_0 \mid \mathcal{D}, \mathcal{C})$.
\algphasegap
    
\STATE Factorise $\mathbf{K}_{**|\bm y} = \mathbf{L}\mathbf{L}^{\intercal}$
       \hfill\COMMENT{one-off $\mathcal{O}(m^3)$ cost; defines $\mathcal{W}^{-1}(\cdot) = \mathbf{L}\cdot + \mathbf{m}_{*|\bm y}$}
\STATE Build decreasing time grid
       $1 = t_0 > t_1 > \cdots > t_T \approx 0$, uniform in log-SNR
       (Appendix~\ref{appendix:snr})
\STATE Draw initial whitened state $\hat{\bm f}_{t_0} \sim \mathcal{N}(\mathbf{0},\mathbf{I}_m)$
       \hfill\COMMENT{$t=1$: pure white noise}
       \algphasegap
       \FOR{$j = 0, 1, \ldots, T-1$}
    \STATE $\Delta t_j \leftarrow t_j - t_{j+1}$,\quad
           $\alpha_j \leftarrow \alpha(t_j)$,\quad
           $\beta_j  \leftarrow \beta(t_j)$
   
           \algphasegap
           \algphasegap
    
           \COMMENT{(i) Draw $S$ samples from 
                    $\hat{\bm f}_0 \mid \hat{\bm f}_{t_j}, \mathcal{D}
                     \sim \mathcal{N}\bigl(\alpha_j\,\hat{\bm f}_{t_j},\,(1-\alpha_j^{2})\mathbf{I}_m\bigr)$; see
                    \eqref{eq:whitened_bridge_distribution_notes}.}
    \FOR{$i = 1, \ldots, S$}
        \STATE $\hat{\bm f}_0^{(i)} \sim
               \mathcal{N}\bigl(\alpha_j\,\hat{\bm f}_{t_j},\,(1-\alpha_j^{2})\mathbf{I}_m\bigr)$
        \STATE $\bm f_0^{(i)} \leftarrow
               \mathbf{L}\,\hat{\bm f}_0^{(i)} + \mathbf{m}_{*|\bm y}$
               \hfill\COMMENT{unwhiten before evaluating the likelihood}
    \ENDFOR
    \algphasegap
    \algphasegap
    
    \COMMENT{(ii) Self-normalised importance weights (numerically stable form)}
    \STATE $\ell^{(i)} \leftarrow \log p\bigl(\mathcal{C}\mid\bm f_0^{(i)}\bigr)$
           for $i=1,\ldots,S$
    \STATE $\bar w^{(i)} \leftarrow
           \exp\!\Bigl(\ell^{(i)}
                - \operatorname{logsumexp}_{r}\ell^{(r)}\Bigr)$
           for $i=1,\ldots,S$
    \algphasegap
    \COMMENT{(iii) Likelihood scores via automatic differentiation
                    through $\bm f_0^{(i)} = \mathbf{L}\hat{\bm f}_0^{(i)} + \mathbf{m}_{*|\bm y}$}
    \STATE $\bm s^{(i)} \leftarrow
           \nabla_{\hat{\bm f}_0^{(i)}}\log p\bigl(\mathcal{C}\mid\bm f_0^{(i)}\bigr)$
           for $i=1,\ldots,S$
    
    \algphasegap
    \COMMENT{(iv) Guided velocity field in whitened space
                   (Eq.~(\ref{eq:whitened}) with MC guidance from Eq.~(\ref{eq:guidance-mc-estimator}))}
    \STATE $\bm v \leftarrow
           -\tfrac{1}{2}\,\beta_j\,\alpha_j
           \sum_{i=1}^{S} \bar w^{(i)}\,\bm s^{(i)}$
    
           \algphasegap
    \COMMENT{(v) Smooth clipping of the step to promote numerical stability}
    \STATE $\bm v \leftarrow
           \dfrac{v_{\max}\,\tanh\!\bigl(\|\bm v\|/v_{\max}\bigr)}
                 {\|\bm v\|}\,\bm v$
    \algphasegap
    \COMMENT{(vi) Explicit Euler step of the probability-flow ODE, integrating from $t=1$ to $t=0$}
    \STATE $\hat{\bm f}_{t_{j+1}} \leftarrow \hat{\bm f}_{t_j} - \Delta t_j\,\bm v$
    \algphasegap
    \ENDFOR
\STATE \textbf{return}
       $\bm f_0 \leftarrow
        \mathbf{L}\,\hat{\bm f}_{t_T} + \mathbf{m}_{*|\bm y}$
       \hfill\COMMENT{unwhiten final state}
\end{algorithmic}
\end{algorithm}

\section{Experimental details for monotonic and bounded Bayesian regression}
\label{app:monotonic}

\paragraph{Task.}
We demonstrate shape-constrained GP regression, conditioning a 1D GP on noisy
observations while simultaneously enforcing that the inferred function is
(i) \emph{monotonically increasing} and (ii) \emph{bounded} between a
known lower and upper envelope.
This setting is representative of problems in dose-response modelling,
reliability analysis, or any domain where monotonicity and range constraints
are required.

\paragraph{Data.}
The ground truth function is
\begin{equation}
    f_{\text{true}}(x) = \frac{1}{3}\!\left[\arctan(20x - 10) - \arctan(-10)\right],
    \quad x \in [0, 1].
\end{equation}
We observe $n = 7$ noise-free evaluations at locations
$x_i = 0.1 + 1/(i+1)$ for $i = 1, \ldots, 7$, generated according to \eqref{eqn:gp_model} with $\bm\Gamma=\sigma^2\mathbf{I}_n$ and measurement-error variance $\sigma^2 = 10^{-10}$ (i.e.\ effectively noiseless observations).

\paragraph{Shape constraints.}
Our target is $f_0(\cdot)$ evaluated on a uniform grid of size $m = 64$ on $[0,1]$; we denote the collection of elements by $\bm{f}_0$. We make two constraints $\mathcal{C} = \{\mathcal{C}_1, \mathcal{C}_2\}$ when making inference on $\bm{f}_0$:

\textit{Monotonicity.}
We require $f'_0(x) \geq 0$ everywhere on $[0,1]$.
This is implemented on our uniform grid using constraints on the forward finite differences
\begin{equation}
    c_i = \frac{f_{0,i+1} - f_{0,i}}{\Delta x}, \quad i = 1, \ldots, m-1,
\end{equation}
where $\Delta x = 1/m$.
We implement these constraints through a probit log-probability model (via a standard-normal CDF relaxation):
\begin{equation}\label{eq:probit}
    \log p(\mathcal{C}_1 \mid \mathbf{f}_0) = \sum_{i=1}^{m-1} \log \Phi\!\left(\frac{c_i}{v}\right),
\end{equation}
where $\mathcal{C}_1 = \{c_1,\dots,c_{m-1}\}$, $\Phi(\cdot)$ is the standard normal CDF, and the sharpness parameter $v = 10^{-4}$. The model heavily penalises for any $c_i$ being negative.

\textit{Boundedness.}
We require $\ell(x) \leq f_0(x) \leq u(x)$ for all $x$, where the lower
and upper envelopes are
\begin{equation}
    \ell(x) = 0, \qquad
    u(x) = \frac{1}{3}\log(30x + 1) + 0.1.
\end{equation}
We implement this constraint on $\bm{f}_0$ by considering the $2m$ margin values $\mathcal{C}_2 = \{(u(x_i) - f_{0,i},\, f_{0,i} - \ell(x_i)): i = 1,\dots,m\}$. These are modelled using the probit relaxation of \eqref{eq:probit} with $v = 10^{-5}$ and $\mathcal{C}_1$ replaced with $\mathcal{C}_2$.

\paragraph{GP prior.}
We use a GP prior with a squared exponential kernel over functions on $[0,1]$ of the form 
$$\text{cov}(f_0(x), f_0(x')) = \tau^2 \exp\left(\frac{-(x-x')^2}{2\kappa^2}\right)$$, 
with fixed length-scale
$\kappa = 0.1$ and scaling parameter $\tau^2 = 0.25$ (assumed known).
The GP evaluated on our grid is conditioned on the 7 observations to obtain the posterior mean
$\mathbf{m}_{*|\bm y}$ and covariance $\mathbf{K}_{**|\bm y}$ on the $m = 64$ grid points, yielding
the base Gaussian predictive distribution $\mathcal{N}(\mathbf{m}_{*|\bm y}, \mathbf{K}_{**|\bm y})$.

\paragraph{Additional results.} Figure \ref{fig:monotone_full} shows the predictions from \textsc{FlowGP} along with those from an unconstrained GP and two projection methods \citep{lin2014bayesian,astfalck2018posterior}. \textsc{FlowGP} closely matches the quality of \cite{astfalck2018posterior}, a bespoke method specifically engineered for this constraint structure whilst \cite{lin2014bayesian} exhibits prior inconsistency.

\begin{figure}
    \centering
    \begin{subfigure}[b]{0.49\textwidth}
        \centering
        \includegraphics[width=\textwidth]{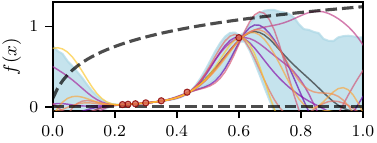}
        \caption{Unconstrained GP}
    \end{subfigure}
    \hfill
    \begin{subfigure}[b]{0.49\textwidth}
        \centering
        \includegraphics[width=\textwidth]{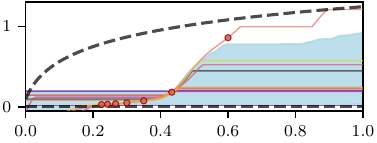}
        \caption{Projected GP \cite{lin2014bayesian} (5ms)}
    \end{subfigure}
    \begin{subfigure}[b]{0.49\textwidth}
        \centering
        \includegraphics[width=\textwidth]{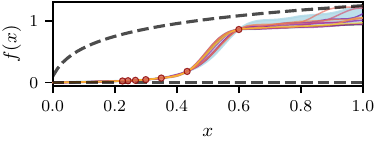}
        \caption{Posterior Projected GP \cite{astfalck2018posterior} (40ms)}
    \end{subfigure}
    \hfill
    \begin{subfigure}[b]{0.49\textwidth}
        \centering
        \includegraphics[width=\textwidth]{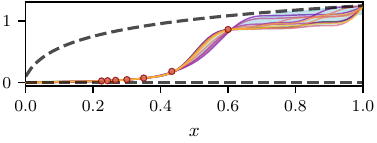}
        \caption{\textsc{FlowGP} (Ours) (200ms)}
    \end{subfigure}
    \caption{Predictive samples from an unconstrained GP \textbf{(a)}, two custom projection methods \textbf{(b)} and \textbf{(c)} \cite{lin2014bayesian, astfalck2018posterior}, and \textsc{FlowGP} \textbf{(d)}, for a bounded monotonic regression problem. All methods are conditioned on seven observations (red dots) and we wish to encode monotonicity and bound constraints (dashed lines). \textsc{FlowGP} remains prior-consistent and matches the quality of the bespoke method, without a substantial increase in runtime. We plot $10$ predictive samples and the 0.05-0.95 quantiles from $100$ samples. Timings are all on an Apple M2 Pro processor with 32GB RAM.}
    \label{fig:monotone_full}
\end{figure}

\section{Experimental details for physics-informed experiments}
\subsection{Evaluation strategy}
\textbf{Evaluating sample paths at new test points.}
\textsc{FlowGP} produces sample paths evaluated on a fixed discretisation grid, whereas the test observations lie at
arbitrary locations $\mathbf{x}^*$. Therefore, we extend our sampled field across new test points by kernel smoothing. In particular, we can apply the standard GP regression formula to get a smoothed field at $\mathbf{x}^{new}$. Of course, we can longer guarantee that our condition will hold at $\mathbf{x}^{new}$, but we can ensure that it is at least consistent with the closed-form GP components. In particular, we extend a sample $\bm f^{(s)}$ defined over a grid $\bm X_{*}$ to a new location via
\begin{equation}
    \hat{\bm f}^{(s)}(\mathbf{x}^{new}) \;=\;
    \mu_{*|y}(\bm x^{\textrm{new}}) \;+\;
    \mathbf{k}(\mathbf{x}^{new}, \bm X_{*})\,
    \bm K_{**|y}^{-1}
    \bigl(\bm f^{(s)} - \bm m_{*|y})\bigr),
\end{equation}
where  $\mu_{*|y}(\cdot)\in \mathbb{R}^{1}$ is the GP's posterior mean at $\mathbf{x}^{new}$ and $\mathbf{k}_{*|y}(\mathbf{x}^*, \bm X_*) \in \mathbb{R}^{1\times m}$ is the vector of
posterior cross-covariances between at $\mathbf{x}^{new}$ and all the grid points $\bm X_*$, and $\textbf{K}_{**|y}$ and $\bm m_{*|y}$ are the covariance and mean of $\bm f_0|\mathcal{D}$.

\textbf{Metrics.} From the ensemble we form a Monte Carlo predictive distribution.
The predictive mean and variance at test point $i$ are
\begin{equation}
    \hat{\mu}_i = \frac{1}{S}\sum_{s=1}^S \hat{f}^{(s)}(\mathbf{x}^*_i),
    \qquad
    \hat{\sigma}^2_i = \frac{1}{S-1}\sum_{s=1}^S
        \bigl(\hat{f}^{(s)}(\mathbf{x}^*_i) - \hat{\mu}_i\bigr)^2 + \sigma^2_\varepsilon,
\end{equation}
where $\sigma^2$ is the observation noise variance (if applicable).
We report two scalar metrics averaged over the $M$ test points:
\begin{align}
    \mathrm{RMSE} &= \sqrt{\frac{1}{M}\sum_{i=1}^M (\hat{\mu}_i - y_i)^2}, \\
    \mathrm{NLPD} &= \frac{1}{M}\sum_{i=1}^M
        \left[
            \tfrac{1}{2}\log(2\pi\hat{\sigma}^2_i)
            + \frac{(y_i - \hat{\mu}_i)^2}{2\hat{\sigma}^2_i}
        \right],
\end{align}
where $y_i$ is the observed target.
RMSE measures point-prediction accuracy, while NLPD (negative log predictive density) also
penalises miscalibrated uncertainty.
Uncertainty in both metrics is reported as $\pm 1$ standard deviation estimated by
bootstrapping over the $1000$ samples.

\label{appendix:physics}
\subsection{Damped Pendulum Experiment}
\label{app:pendulum}

\paragraph{Task.}
We consider the problem of inferring the angular displacement trajectory
$\theta(t)$ of a damped nonlinear pendulum from noisy sparse observations,
while simultaneously enforcing that the inferred trajectory satisfies the
governing equation of motion.
The pendulum dynamics are described by the second-order ODE
\begin{equation}
    {\theta}^{\prime\prime}(t) + \sin\!\bigl(\theta(t)\bigr) + \beta\,{\theta}^\prime(t) = 0,
    \label{eq:pendulum}
\end{equation}
where $\beta = 0.2$ is the damping coefficient.
The time horizon is $t \in [0, 30]$\,s (note that here and below we abuse notation somewhat and let `$t$' also index our input space).

\paragraph{Data.}
We use a fixed train/test split generated from a numerical solution of
\cref{eq:pendulum}.
Training observations are noisy evaluations of $\theta(\cdot)$ at irregularly
spaced time points.
The test set is evaluated on a subsampled grid of held-out time points. Both training and test data are as provided in the \textsc{PHYSS} \cite{hamelijnck2024physics}.

\paragraph{Experimental settings.}
We run two settings with different purposes:
\begin{itemize}
    \item \textbf{Qualitative / plotting setting.}
    Measurement-error variance $\sigma^2 = 0.15^2$.
    The GP posterior and samples are evaluated on a grid of $m = 250$ points.
    We generate 25 predictive samples.
    This setting is used to produce visualisations.

    \item \textbf{Quantitative / metrics setting.}
    Measurement-error variance $\sigma^2 = 0.01^2$.
    The grid is reduced to $m = 125$ points, and 1000 predictive samples are drawn to obtain reliable estimates
    of predictive means and variances.
    This setting is used to compute RMSE and NLPD.
\end{itemize}

\paragraph{GP prior.}
We place a GP prior over trajectories $f_0 : [0,1] \to \mathbb{R}$
(time is normalised by dividing by $30$) using a scaled RBF kernel
\begin{equation}
    \text{cov}(f_0(t), f_0(t')) = \tau^2 \exp\!\left(-\frac{(t-t')^2}{2\kappa^2}\right),
\end{equation}
with hyperparameters $\tau^2$ and $\kappa$ optimised together with an affine mean function by maximising the marginal likelihood using just $\mathcal{D}$,.
As in the previous experiment, the GP predictive mean $\mathbf{m}_{*|\bm y}$ and covariance $\mathbf{K}_{**|\bm y}$ on the evaluation grid
are used to construct the base Gaussian predictive distribution $\mathcal{N}(\mathbf{m}_{*|\bm y}, \mathbf{K}_{**|\bm y})$.

\paragraph{Physics-informed conditioning.}
The ODE constraint~\eqref{eq:pendulum} is encoded as a Gaussian likelihood
over the pointwise residuals of our target function $f_0(\cdot)$ on our prediction grid $t_1,\dots,t_m$:
\begin{equation}
    r(t_j)
    = {f_0}^{\prime\prime}(t_j) + \sin(f_0(t_j)) + \beta\,{f_0}^{\prime}(t_j), \quad j = 2,\dots,m-1,
\end{equation}
with a likelihood standard deviation parameter $\sigma_\text{phys} = 10^{-10}$. Note that, while the likelihood function is Gaussian, the constraints it encodes are nonlinear. We evaluate the derivatives using central finite differences (see Appendix \ref{app:burgers} for an example).
Guidance is applied during probability-flow ODE denoising
($1000$ steps, $5$ MC samples per step, whitened representation).

\paragraph{Evaluation.}
We form the predictive mean and variance
from predictive samples, adding measurement error with variance $\sigma^2$ to the
sample variance in order to get the predictive variance of test data. We report results using RMSE and NLPD.

\subsection{Allen--Cahn Equation Experiment}
\label{app:allen_cahn}

\paragraph{Task.}
We consider the problem of inferring the solution field $u(x, t)$ of the
Allen--Cahn equation from noisy, partial observations in an early-time window,
while enforcing the PDE and symmetry boundary conditions over the full domain.
The governing equation is
\begin{equation}
    \frac{\partial u(x,t)}{\partial t}
    = \varepsilon \frac{\partial^2 u(x,t)}{\partial x^2}
    + 5u(x,t) - 5u(x,t)^3,
    \label{eq:allen_cahn}
\end{equation}
with $\varepsilon = 10^{-5}$, $x \in [-1, 1]$, $t \in [0, 1]$.

\paragraph{Data.}
Training observations are drawn from the
early-time regime $t < 0.28$ (256 randomly subsampled points) and the test
set contains 1000 randomly subsampled points from $t > 0.28$. Both training and test data splits are as provided in the \textsc{PHYSS} \cite{hamelijnck2024physics}.

\paragraph{Experimental settings.}
We run two settings with different purposes:
\begin{itemize}
    \item \textbf{Qualitative / plotting setting.}
    Measurement-error variance $\sigma^2 = (0.01)^2$.
    The field is discretised on a $H \times W = 100 \times 100$ grid
    ($m = 10{,}000$).
    We draw 5 predictive samples.
    This setting is used to produce visualisations.

    \item \textbf{Quantitative / metrics setting.}
    Effectively noiseless observations ($\sigma^2 = 10^{-10}$).
    The grid is reduced to $H \times W = 50 \times 20$
    ($m = 1000$).
    We draw 10 predictive samples.
    This setting is used to compute RMSE and NLPD.
\end{itemize}

\paragraph{GP prior.}
We use a squared-exponential covariance function 2D field $f_0 : [-1,1]\times[0,1] \to \mathbb{R}$
\begin{equation}
    \text{cov}\!\left(f_0(x,t),f_0(x',t')\right)
    = \tau^2\exp\!\left(-\frac{(x-x')^2}{2\kappa_x^2} - \frac{(t-t')^2}{2\kappa_t^2}\right),
\end{equation}
with %
the length scales and variances fitted by maximising the marginal likelihood using just $\mathcal{D}$.
As in the previous experiments, the GP predictive mean $\mathbf{m}_{*|\bm y}$ and covariance $\mathbf{K}_{**|\bm y}$ on the evaluation grid
are used to construct the base Gaussian predictive distribution $\mathcal{N}(\mathbf{m}_{*|\bm y}, \mathbf{K}_{**|\bm y})$.

\paragraph{Physics-informed conditioning.}
Samples are conditioned jointly on the Allen--Cahn PDE residual (evaluated at
interior spatial rows, via central finite differences in both $x$ and $t$)
and on symmetric boundary conditions enforcing $f_0(-1,t) = f_0(1,t)$ and
matching first spatial derivatives at the boundaries ( see Appendix \ref{app:burgers} for an example).
Both conditions use $\sigma_\text{phys} = 10^{-5}$, and are combined as a product when constructing $p(\mathcal{C} \mid \bm f_0)$.

\paragraph{Evaluation.}
For each test point $(x_j, t_j)$ we %
form the predictive mean and variance from the predictive samples, before adding
observation noise $\sigma^2$ to the predictive variance. We report our results using RMSE and NLPD.

\subsection{Viscous Burgers' Equation Experiment}
\label{app:burgers}

\paragraph{Task.}
We consider the problem of inferring the velocity field $u(x, t)$ of the
one-dimensional viscous Burgers' equation from a small number of observations
of the initial condition, while enforcing that the inferred field satisfies the
PDE and homogeneous Dirichlet boundary conditions over the full spatio-temporal
domain (as considered in \cite{chen2021solving}).
The governing equation is
\begin{equation}
    \frac{\partial u(x,t)}{\partial t} + u\frac{\partial u(x,t)}{\partial x}
    = \nu \frac{\partial^2 u(x,t)}{\partial x^2},
    \label{eq:burgers}
\end{equation}
where $\nu = 0.02$ is the kinematic viscosity.
The spatial domain is $x \in [-1, 1]$ and the time domain is $t \in [0, 1]$.

\paragraph{Data and experimental variants.}
We assume we have noisy observations of the initial condition that are generated from the true profile
\begin{equation}
    u(x, 0) = -\sin\!\bigl(\pi(2x - 1)\bigr),
\end{equation}
 at uniformly spaced interior points.
We consider two experimental settings:

\begin{itemize}
    \item \textbf{Noisy, scarce data.} $n = 5$ observations are
    used, each corrupted by independent Gaussian measurement error with variance
    $\sigma^2 = (0.01)^2$. This setting tests the ability to recover the
    PDE solution from very few, noisy measurements.

    \item \textbf{Noiseless, dense data.} $n = 100$ observations
    are used, with effectively zero measurement error ($\sigma^2 = 10^{-12}$).
    This setting tests physics enforcement when the initial condition is
    almost exactly known on a dense grid.
\end{itemize}

In both cases the observations are 2D input pairs $(x_j, 0)$ covering only the
initial time slice; the model must extrapolate the dynamics forward in time by
satisfying the PDE constraint.
Ground-truth solutions at three evaluation snapshots
$t \in \{0.2, 0.5, 0.8\}$ are used for quantitative comparison.

\paragraph{GP prior.}
We place a Gaussian process prior over the 2D field
$f_0 : [-1, 1] \times [0, 1] \to \mathbb{R}$.
We use a squared exponential kernel
\begin{equation}
    \text{cov}\!\left(f_0(x,t),f_0(x',t')\right)
    = \tau^2\exp\!\left(-\frac{(x-x')^2}{2\kappa_x^2} - \frac{(t-t')^2}{2\kappa_t^2}\right),
\end{equation}

with fixed length-scales $\kappa_x = 0.025$, $\kappa_t = 0.3$, and output scale
$\tau^2 = 1.0$ (assumed known, following the benchmark protocol of \cite{chen2021solving}).

We construct the GP predictive mean $\mathbf{m}_{*|\bm y}$ and covariance $\mathbf{K}_{**|\bm y}$ on a uniform
$H \times W = 50 \times 20$ grid, giving a base Gaussian predictive distribution $\mathcal{N}(\mathbf{m}_{*|\bm y}, \mathbf{K}_{**|\bm y})$ with
$m = HW = 1000$.

\paragraph{Physics-informed conditioning.}
We condition samples jointly on the Burgers' PDE residual and on homogeneous
Dirichlet boundary conditions.

\textit{PDE residual.}
Given a candidate field $\mathbf{f}_0 \in \mathbb{R}^{H \times W}$ with spacings
$\Delta x = 2/(H-1)$ and $\Delta t = 1/(W-1)$, we compute derivatives via central
finite differences at interior points:
\begin{align}
    \left.\frac{\partial f_0}{\partial t}\right|_{i,j}
    &\approx \frac{f_{0,i,j+1} - f_{0,i,j-1}}{2\Delta t},
    \qquad j = 1,\ldots,W-2, \\
    \left.\frac{\partial f_0}{\partial x}\right|_{i,j}
    &\approx \frac{f_{0,i+1,j} - f_{0,i-1,j}}{2\Delta x},
    \qquad i = 1,\ldots,H-2, \\
    \left.\frac{\partial^2 f_0}{\partial x^2}\right|_{i,j}
    &\approx \frac{f_{0,i+1,j} - 2f_{0,i,j} + f_{0,i-1,j}}{\Delta x^2}.
\end{align}
The pointwise PDE residual is
\begin{equation}
    r(x_j, t_j)
    = \frac{\partial f_0}{\partial t}\bigg|_{i,j}
    + f_{0,i,j}\frac{\partial f_0}{\partial x}\bigg|_{i,j}
    - \nu \frac{\partial^2 f_0}{\partial x^2}\bigg|_{i,j}.
\end{equation}

\textit{Boundary conditions.}
Homogeneous Dirichlet conditions $f_0(-1, t) = f_0(1, t) = 0$ for all $t$ are
enforced by pinning the top ($i=0$) and bottom ($i=H-1$) spatial rows to zero.

Both conditions are modelled as Gaussian likelihoods with
$\sigma_\text{phys} = 10^{-5}$ (PDE) and $\sigma_\text{bc} = 10^{-6}$
(boundary), combined as a product when constructing $p(\mathcal{C} \mid \bm f_0)$. Unlike all other experiments in this paper, we found that \textsc{FlowGP} required $10,000$ ODE steps (as opposed to $1,000$) to resolve the resulting flow.

\paragraph{Evaluation.}
We evaluate the conditioned samples at $t \in \{0.2, 0.5, 0.8\}$, reporting
RMSE and NLPD at each snapshot (pooling all spatial points across
all three temporal snapshots) for both experimental variants. Results shown in Table \ref{tab:burgers}, where we follow the same formatting as in Table \ref{tab:physics}.

\begin{table}[h]
    \centering
\begin{tabular}{@{}p{0.8cm}rrrr@{}}
\toprule
\textbf{Model} & \textbf{RMSE} & \textbf{NLPD} & \textbf{Time} \\
\midrule
$\textrm{Kernel}$       & \textbf{0.01} & $\textrm{N}/\textrm{A}$ & $\mathbf{4\!\cdot\!10^{1}}$ \\
\midrule
$\textrm{FLOW}$         & 0.04 (0.00) & \textbf{-1.82 (0.04)} & $8\!\cdot\!10^{1}$ \\
\midrule
$\textrm{FLOW}_{\textrm{W}}$ & 0.03 (0.00) & -1.35 (0.12)& $8\!\cdot\!10^{1}$ \\
\end{tabular}
\begin{tabular}{@{}p{0.8cm}rrrr@{}}
\toprule
\textbf{Model} & \textbf{RMSE} & \textbf{NLPD} & \textbf{Time} \\
\midrule
$\textrm{Kernel}$       & 0.25 & $\textrm{N}/\textrm{A}$ & $\mathbf{4\!\cdot\!10^{1}}$ \\
\midrule
$\textrm{FLOW}$         & 0.30 (0.01) & -0.81 (0.05) & $8\!\cdot\!10^{1}$ \\
\midrule
$\textrm{FLOW}_{\textrm{W}}$ & \textbf{0.22} (0.00) & \textbf{-1.44 (0.03)} & $8\!\cdot\!10^{1}$ \\
\end{tabular}
    \caption{The whitened formulation of \textsc{FlowGP} matches existing physics-obeying deterministic kernel regression methods \citep{chen2021solving} in accuracy on the Burgers benchmarks with the additional benefit of providing confidence intervals. \textbf{left:} dense low-noise regime. \textbf{right:} sparse high-noise regime where the kernel approach degrades. We used the authors' open-source implementation, timing both theirs and our methods on our Nvidia A4500 workstation.}
    \label{tab:burgers}
\end{table}

\section{Bayesian Optimisation Experiments}
\label{appendix:BO}

\paragraph{Problem setting} We consider the BOPE setup of \citep{wang2025bayesian}, and provide a short description below. For more details, we refer readers to the original paper, given that we consider exactly the same setting (with only a few minor differences introduced by \textsc{FlowGP}, which we discuss in the next paragraph). In \citep{wang2025bayesian}, an expensive and black-box function $\bm{f}_{\text{true}}: \mathbb{R}^d \to \mathbb{R}^k$ maps a decision variable $\bx \in \mathcal{X} \subseteq \mathbb{R}^d$ to $k$ objective values, and a decision maker (DM) has an unknown utility $g_{\text{true}}: \mathbb{R}^k \to \mathbb{R}$ over these outputs. The goal is to find
\begin{equation*}
\bx^* \in \operatorname*{arg\,max}_{\bx \in \mathcal{X}} \, g_{\text{true}}(\bm{f}_{\text{true}}(\bx)).
\end{equation*}
The procedure alternates between two stages. In \emph{experimentation}, the experimenter selects $\bx_t$ and observes $\bm{f}_{\text{true}}(\bx_t)$, yielding a dataset $\mathcal{D} := \{(\bx_i, \bold{y}_i)\}_{i=1}^n$. In \emph{preference exploration}, the DM is shown a pair of previously observed outputs $(\bold{y}_{1,t}, \bold{y}_{2,t})$ and returns $p_t \in \{-1, +1\}$. Restricting comparisons to observed outputs, rather than allowing arbitrary or unattainable pairs as in \citep{lin2022preference}, avoids presenting the DM with options that cannot be realised. We model DM noise as utility noise:
\begin{equation*}
p_t = 2 \cdot \mathbbm{1}_{g_{\text{true}}(\bold{y}_{1,t}) - g_{\text{true}}(\bold{y}_{2,t}) + \epsilon > 0} - 1, \quad \epsilon \sim \mathcal{N}(0, \sigma_{\text{noise}}^2),
\end{equation*}
so error probability scales with the utility gap rather than being constant. After $m$ comparisons, the preference data is $\mathcal{P}_t := \{(\bold{y}_{1,i}, \bold{y}_{2,i})\}_{i=1}^m$ with responses $\mathcal{R}_t := \{p_i\}_{i=1}^m$. The two stages alternate until a stopping condition is met, and the setting supports batch acquisition of both $\bx_t$ and comparison pairs.
 
\paragraph{Experimental details} For the experiments in \Cref{fig:fullbo}, we rely on the same set-up and codebase\footnote{\url{https://github.com/HanyangHenry-Wang/BOPE-MoNNE}} as \citep{wang2025bayesian}, modifying only the parts necessary for our \textsc{FlowGP} model to work. Specifically, we rely on the same acquisition function, namely 
\begin{equation*}
    \text{qNEIUU}(\bx_{1:q}) = \frac{1}{M} \frac{1}{N} \sum_{j=1}^{M} \sum_{k=1}^{N} \left\{ \max g^j\left(\bm{f}^k(\bx_{1:q})\right) - \max g^j\left(\bm{f}^k(\bold{X}_n)\right) \right\}^+,
\end{equation*}
where $\bm{f}$ is the same GP surrogate as in \citep{wang2025bayesian}. For $g$, we replace their custom MoNNE architecture by our \textsc{FlowGP}, where the underlying GP is the \texttt{PairwiseGP} model from \texttt{BoTorch} \cite{balandat2020botorch,chu2005preference} and $\mathcal{C}$ encodes monotonicity with respect to all dimensions of $\bx$. After building the two surrogate models $\bm{f}$ and $g$, qNEIUU is computed by Monte Carlo simulation,
where $\bx_{1:q} \in \mathcal{X}^q$ (with $q=1$ in this paper), $\bold{X}_n$ is the tensor of evaluated solutions, $\{\cdot\}^+$ denotes the positive part, $M$ is the number of samples from $g$, $N$ is the number of predictive samples from $\bm{f}$, $g^j(\cdot)$ is the $j$-th sample from $g$, and $\bm{f}^k(\cdot)$ is the $k$-th predictive sample from $\bm{f}$.

For the preference-exploration stage, we do not rely on the more complicated IEUBO presented in \citep{wang2025bayesian}, and instead rely simply on EUBO \citep{lin2022preference}:
\begin{equation*}
    \text{EUBO}(\bold{y}_{1,i}, \bold{y}_{2,i}) = \frac{1}{M} \sum_{j=1}^{M} \max\left\{ g^j(\bold{y}_{1,i}), g^j(\bold{y}_{2,i}) \right\}, \quad \bold{y}_{1,i}, \bold{y}_{2,i} \in \{\bold{y}_i \mid (\bold{x}_i, \bold{y}_i) \in \mathcal{D}_n\}.
\end{equation*}

Since the \textsc{FlowGP} model does not easily allow for gradient-based optimisation of the acquisition function, we instead evaluate the acquisition function on a set of candidate points, from which we select the best one. To generate this set of candidate points, we rely on a technique similar to the state-of-the-art ``Vanilla BO'' method \citep{hvarfner2024vanilla}, which samples points both globally across the search space $\mathcal{X}$ and locally (as a Gaussian) around the incumbent.

\begin{figure}
\centering
\begin{subfigure}[b]{0.16\textwidth}
    \centering
    \includegraphics[width=\textwidth]{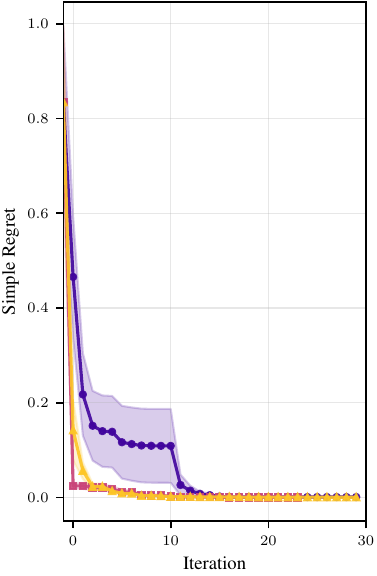}
    \caption{}
\end{subfigure}
\hfill
\begin{subfigure}[b]{0.16\textwidth}
    \centering
    \includegraphics[width=\textwidth]{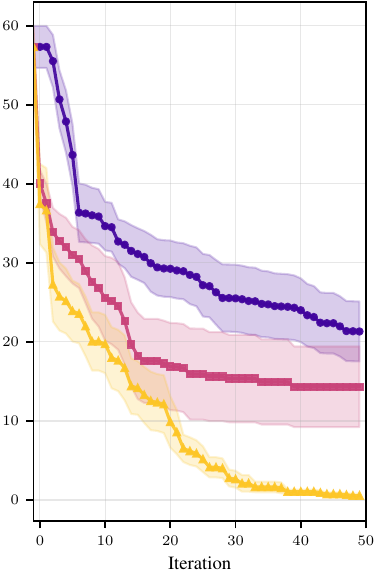}
    \caption{}
\end{subfigure}
\hfill
\begin{subfigure}[b]{0.16\textwidth}
    \centering
    \includegraphics[width=\textwidth]{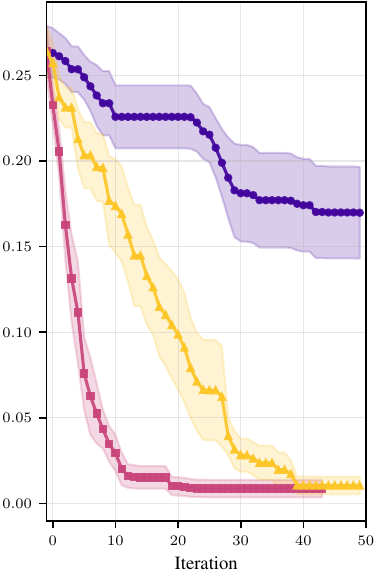}
    \caption{}
\end{subfigure}
\hfill
\begin{subfigure}[b]{0.16\textwidth}
    \centering
    \includegraphics[width=\textwidth]{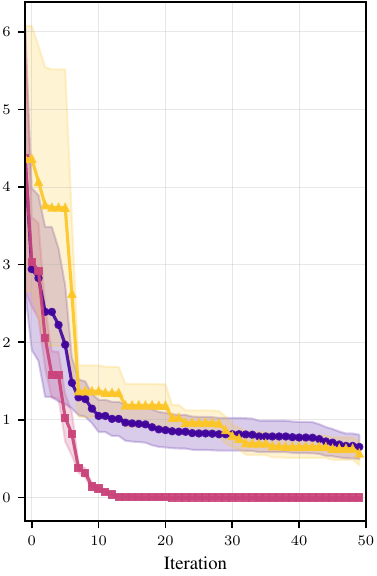}
    \caption{}
\end{subfigure}
\hfill
\begin{subfigure}[b]{0.16\textwidth}
    \centering
    \includegraphics[width=\textwidth]{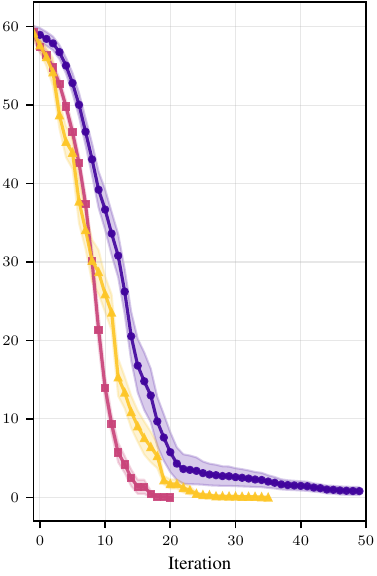}
    \caption{}
\end{subfigure}
\hfill
\begin{subfigure}[b]{0.16\textwidth}
    \centering
    \includegraphics[width=\textwidth]{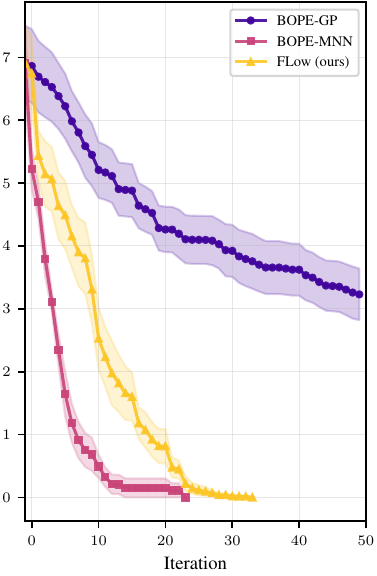}
    \caption{}
\end{subfigure}

\caption{
On all six Bayesian Optimisation with Preference Exploration (BOPE) problems considered in \citep{wang2025bayesian}, our \textsc{FlowGP} model improves upon the underlying preference-learning GP \citep{chu2005preference} by easily incorporating monotonicity, and often matches the problem-specific and carefully-engineered Monotonic Neural Network Ensemble (MoNNE) from \citep{wang2025bayesian}. {\bf (a)} DTLZ2 (3D) + Linear (2D). {\bf (b)} OSY (6D) + Quadratic (8D). {\bf (c)} VehicleSafety (5D) + KumaraswamyCDF (3D). {\bf (d)} VLMOP3 (2D) + Exponential (3D). {\bf (e)} ZDT1 (10D) + LinearExponential (2D). {\bf (f)} CarCabDesign (7D) + CobbDouglas (9D).
}
\label{fig:fullbo}
\end{figure}

\section{LLM-Guided Process Experiment}
\label{app:llm}

\paragraph{Task.}
Our final experiment demonstrates a \emph{product-of-experts} construction
that combines a GP stochastic-process prior with marginal predictive densities
derived from a large language model (LLM).
The LLM is queried with a natural-language description of the time series and
returns a discrete probability distribution over output bins at each input
location; this constitutes a text-conditioned ``expert'' density.

\paragraph{Product-of-experts formulation.}
Let $p(\bm f_0 \mid \mathcal{D})$ denote the GP predictive density and
$q(\bm f_0\mid \mathcal{C})$ the LLM-derived density conditioned on text prompt $\mathcal{C}$.
We define the target distribution as the normalised product of experts
\begin{equation}
    \pi(\bm f_0 \mid \mathcal{D}, \mathcal{C})
    \propto \,p(\bm f_0 \mid \mathcal{D})\,q(\bm f_0 \mid \mathcal{C}).
    \label{eq:GPexpert}
\end{equation}
Here, we are following exactly the setup of \cite{biggs2026llm} and use the LLM-derived density provided by their code.  Note that this product-of-experts formation is an alternative way to include additional knowledge into a GP, distinct from the likelihood conditioning used for all our other experiments.

We refer to their paper for details about the specific prompting strategy, number of tokens, temperature settings, and any post-processing of the Qwen3.5 \cite{team2026qwen3} LLM's output.

\paragraph{Conditioning implementation.}
The LLM scores $q(\bm f_0 \mid \mathcal{C})$ are pre-computed offline and stored as
per-timestep histograms over discrete output bins. In order to provide a smooth likelihood for our framework, we evaluate the log-likelihood of the current sample
under a kernel-smoothed version of the LLM histogram density and add its
gradient to the unconditional score. 

Concretely, let $\bm f_0 \in \mathbb{R}^m$ be the current sample.
For each grid point $j$ and bin $k$ with edges $[b_k^-, b_k^+]$ and LLM
probability mass $p_k^{(j)}$, we compute
\begin{equation}
    \log {q}(\bm f_0 \mid \mathcal{C})
    = \sum_{j=1}^{m} \frac{1}{b_j^+ - b_j^-} \log
    \sum_{k} p_k^{(j)}
    \left[\Phi\!\left(\frac{b_k^+ - f_{0,j}}{\nu}\right)
         - \Phi\!\left(\frac{b_k^- - f_{0,j}}{\nu}\right)\right],
\end{equation}
where $\Phi(\cdot)$ is the standard normal CDF and $\nu = 0.5$ is a smoothing
bandwidth parameter. This replaces the discrete histogram with a continuous, differentiable density
by treating each bin's mass as a Gaussian mixture component.
Gradients of this expression with respect to $\bm f_0$ are computed
analytically via the Gaussian PDF, so no backpropagation through the LLM is
required at any point. We build bins as all the numbers that are 3 significant figures and within the problem range (see below). A \texttt{BoundedCondition} is combined multiplicatively to keep samples within
physically plausible output ranges.

\paragraph{Datasets and text prompts.}
We evaluate on three qualitative scenarios, each defined by a natural-language
prompt, the Qwen3.5 \cite{team2026qwen3} LLM, and a corresponding GP prior:
\begin{itemize}
    \item \textbf{Bankruptcy.} \textit{``A small UK company's daily stock
    price in GBP over 50 trading days. On day 30 the company enters
    compulsory liquidation and is permanently delisted from the London
    Stock Exchange.''} squared exponential kernel ($\kappa = 5/50$, $\tau^2 = 2.0$),
    output bounded to $[0, 10]$.

    \item \textbf{Stable.} \textit{``A small UK company's daily stock price
    in GBP over 50 trading days. The company has no notable news during
    the period, but performs well. Ending with a value over 5 GBP per
    share.''} Same squared exponential kernel as above, output bounded to $[0, 10]$.

    \item \textbf{San Diego (precipitation).} \textit{``Monthly average
    precipitation in San Diego in inches over 50 months, starting in
    January.''} Product of a squared exponential kernel and a periodic kernel (period $= 12/50$,
    $\kappa = 1.0$, $\tau^2 = 2.0$), output bounded to $[0, 4]$.

\end{itemize}
All three scenarios use $m = 50$ output time steps, 100 predictive samples,
two weak anchor observations at $t \in \{1/50, 2/50\}$ to initialise the
GP posterior, and run for $1{,}000$ denoising steps with $1$ MC sample per
step in the whitened representation.

\section{Ablation Study: Guidance Method and Discretisation}
\label{appendix:ablate}

We ablate two axes of our method on the
monotonicity-and-boundedness task in Figure \ref{fig:whitenedablate}:
(i) the guidance estimator and
(ii) the number of denoising steps $T$.
All combinations are evaluated qualitatively by inspecting whether
$S = 1{,}000$ posterior samples respect the monotonicity and
boundedness constraints and display appropriate posterior variation.

\paragraph{Guidance estimators.}
We compare four estimators for the conditional score term
$\nabla_{\mathbf{f}_t} \log p(\mathcal{C} \mid \mathbf{f}_t)$:

\begin{itemize}
    \item \textbf{DPS}~\citep{chung2023diffusion} estimates the guidance term with first-order approximations {${\nabla_{\bm f_t}\log p(\mathcal{C} \mid \E[\bm f_0 \mid \bm f_t])}$}, where $\E[\bm f_0 \mid \bm f_t] = \frac{1}{\alpha(t)}\left(\bm f_t + (1 - \alpha(t)^2)) s(\bm f_t, t)\right)$.

    \item \textbf{MPGD}~\citep{manifold}  uses $\nabla_{\bm f_0}\log p(\mathcal{C} \mid \E[\bm f_0 \mid \bm f_t])$, avoiding differentiation through the denoising process.

    \item \textbf{MC (ours).} Importance-weighted MC gradient using $S \in
    \{1, 10, 100\}$ samples as described in
    \cref{app:method}.

    \item \textbf{MC-Fisher (ours).} An alternative MC estimator that
    forms the guidance direction as the importance-weighted deviation of
    the denoised samples from their mean, using $S \in
    \{1, 10, 100, 1000\}$
\end{itemize}

\textbf{Non-MC approaches underrepresent predictive variance.}
Figure \ref{fig:whitenedablate} shows that both DPS and MPGD produce samples that collapse to near-deterministic
trajectories, dramatically underrepresenting the true posterior uncertainty.
Because both methods condition on a single point estimate of $\bm f_0$
rather than integrating over the denoising predictive distribution, the guidance signal
is overconfident and pushes all samples towards a narrow mode.
The resulting sample sets show almost no spread, failing to capture the
breadth of constraint-satisfying functions consistent with the observations.

\textbf{Our MC estimator works well even with a single sample.}
The importance-weighted MC estimator produces diverse, well-constrained
samples already at $S = 1$.
Even a single draw from the denoising predictive distribution is sufficient to obtain an
unbiased direction that preserves predictive variance, and quality
improves gracefully as $S$ increases.
This makes the MC estimator both practical and accurate in the regime
relevant to our experiments.

\textbf{MC-Fisher requires many samples.}
The Fisher-based MC estimator is substantially less sample-efficient.
With $S < 1000$ the resulting samples still collapse and exhibit poor
diversity, and even at $S = 1000$ the sample spread is only beginning
to look reasonable.
We therefore favour the standard MC estimator in all main experiments.

\paragraph{Number of steps.}
With $T = 10$ steps the ODE discretisation is too coarse and constraint
violations are common regardless of estimator.
At $T = 100$ the MC estimator begins to produce satisfactory results.
$T = 1000$ steps consistently yields smooth, well-constrained samples
under our MC estimator and this choice is used in all main experiments, unless otherwise stated.

\begin{figure}
\centering
\includegraphics[width=0.8\linewidth]{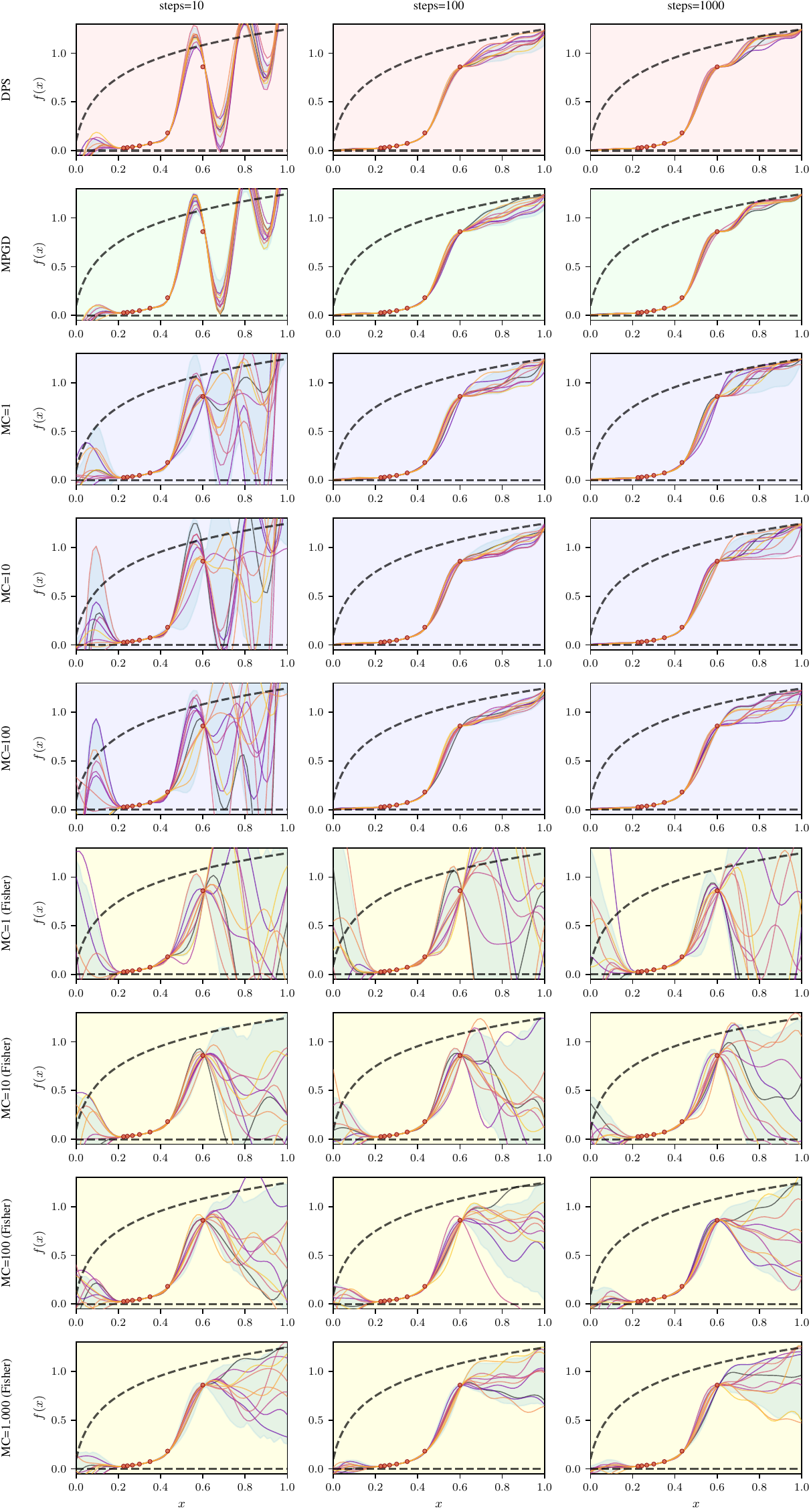}
\caption{Additional experiments on the monotonic and bounded regression problem, showing the effect of different time discretisations (left to right), for guidance estimation strategies (top to bottom), including DPS (red), MPGD (green), our MC estimator (blue), and our MC-Fisher estimator (yellow).}
\label{fig:whitenedablate}
\end{figure}

\end{document}